\documentclass[11pt]{article} 
\usepackage{url}
\usepackage{smile}
\usepackage{graphicx} 
\usepackage{algorithm}
\usepackage{algorithmic}
\usepackage{todonotes}
\usepackage{epstopdf}
\usepackage[colorlinks, linkcolor=blue, anchorcolor=blue, citecolor=blue]{hyperref}
\usepackage[margin=1in]{geometry}
\usepackage[normalem]{ulem}
\usepackage[export]{adjustbox}
\usepackage{mathtools, cuted}
\usepackage{natbib}

\usepackage{kpfonts}
\DeclareMathAlphabet{\mathsf}{OT1}{cmss}{m}{n}
\SetMathAlphabet{\mathsf}{bold}{OT1}{cmss}{bx}{n}

\providecommand{\norm}[1]{\|#1\|}

\begin{document}

\title{\huge Noisy Gradient Descent Converges to Flat Minima for Nonconvex Matrix Factorization}

\author{Tianyi Liu, Yan Li, Song Wei, Enlu Zhou and Tuo Zhao\thanks{T. Liu, Y.Li, S.Wei, E. Zhou, and T. Zhao are affiliated with School of Industrial and Systems Engineering at Georgia Tech; Tianyi Liu and Tuo Zhao is the corresponding author; Email: \{tianyiliu,tourzhao\}@gatech.edu.}}
\date{}

\maketitle
\begin{abstract}
Numerous empirical evidences have corroborated the importance of noise in nonconvex optimization problems.
The theory behind such empirical observations, however, is still largely unknown. 
This paper studies this fundamental problem through investigating the nonconvex rectangular matrix factorization problem, which has infinitely many global minima due to rotation and scaling invariance.
Hence, gradient descent (GD) can converge to any optimum, depending on the initialization. 
In contrast, we show that a perturbed form of GD with an arbitrary initialization converges to a global optimum that is uniquely determined by the injected noise. 
Our result implies that the noise imposes implicit bias towards certain  optima.
Numerical experiments are provided to support our theory.
\end{abstract}
\section{Introduction}
Nonconvex optimization has been widely adopted in various domains, including  image recognition  \citep{hinton2012deep, krizhevsky2012imagenet}, Bayesian graphical models \citep{jordan2004graphical, attias2000variational}, recommendation systems \citep{salakhutdinov2007restricted}, etc. 
Despite the fact that solving a nonconvex problem is generally difficult, empirical evidences have shown that  simple first order algorithms such as stochastic gradient descent (SGD), are able to solve a majority of the aforementioned nonconvex problems efficiently.
The theory behind these empirical observations, however, is still largely unexplored. 

In classical optimization literature, there have been fruitful results on characterizing the convergence of  SGD to first-order stationary points for nonconvex problems.
However, these types of results fall short of explaining the empirical evidences that SGD often converges to global minima for a wide class of nonconvex problems used in practice. 
More recently, understanding the role of noise in the algorithmic behavior of SGD has received significant attention.
For instance, \citet{jin2017escape} show that a perturbed form of gradient descent is able to escape from strict saddle points and converge to second-order stationary points (i.e., local minima).  
\citet{zhou2019towards} further show that noise in the update can help SGD to escape from spurious local minima and converge to the global minima.
We argue that, 
despite all these recent results on showing the convergence of SGD to global minima for various nonconvex problems,  
there is still an important question yet to be addressed.
Specifically, the  convergence of SGD  in the presence of multiple global minima remains uncleared for many important nonconvex problems. 
For example, for over-parameterized neural networks, it has been shown that the nonconvex objective has multiple global minima \citep{kawaguchi2016deep}, only few of which can yield good generalization.
In addition, \citet{allen2018convergence}  show that for an over-parameterized network,  SGD can converge to a global minimum in polynomial time.
Combining with the empirical successes of training over-parameterized neural networks with SGD, these results strongly advocate that SGD not only can solve the nonconvex problem efficiently, but also implicitly biases towards solutions with good generalization ability.
Motivated by this, this paper aims to provide more theoretical insights to the following question:
\begin{center}
\textbf{\emph{Does noise impose implicit bias towards certain minimizer
in nonconvex optimization problems?}}
\end{center}

We answer this question through investigating a simple yet non-trivial problem -- nonconvex matrix factorization,  which serves as an important foundation for a wide spectrum of problems such as matrix sensing \citep{bhojanapalli2016dropping,zhao2015nonconvex,chen2015fast,tu2015low}, matrix
completion \citep{keshavan2010matrix,hardt2014understanding,zheng2016convergence}, and deep linear networks \citep{ji2018gradient, gunasekar2018implicit}. 
Given a matrix $M \in \RR^{d_1 \times d_2} $, the nonconvex matrix factorization aims to solve:
\begin{align}\label{mf_obj}
\min_{X\in \RR^{d_1 \times r},Y \in \RR^{d_2 \times r}} \frac{1}{2}\norm{XY^\top - M}_{\rm{F}}^2.
\end{align}
Despite its simplicity, \eqref{mf_obj} possesses several intriguing landscape properties: nonconvexity of the objective, all the saddle points satisfy the strict saddle property, and infinitely many global optima due to scaling and rotational invariance.
Specifically, for any pair of global optimum $(X^*, Y^*)$, $(\alpha X^*, \frac{1}{\alpha} Y^*)$ and $(X^*R, Y^*R)$ are also global optima for any non-zero constant $\alpha$ and rotation matrix $R \in \RR^{r \times r}$. This is different from symmetric matrix factorization, which only possesses rotational invariance. 
The scaling and rotational invariance also imply that the global minima of \eqref{mf_obj} are connected, a landscape property that is also shared by deep neural networks \citep{nguyen2017loss, draxler2018essentially,nguyen2018loss,venturi2018spurious,garipov2018loss,liang2018understanding,nguyen2019connected,kuditipudi2019explaining}.

Nonconvex matrix factorization  \eqref{mf_obj} has been recently studied by \citet{du2018algorithmic}, with  focus on the algorithmic behavior of gradient descent (GD).
Their results reveal an interesting algorithmic regularization imposed by gradient descent: (i) gradient flow (GD with an infinitesimal step size) has automatic balancing property, i.e., the difference of the squared norm $\norm{X}_{\rm{F}}^2 - \norm{Y}_{\rm{F}}^2$ stays constant during training. (ii) for properly chosen step size, GD converges asymptotically for rank-r case,  linearly for rank-1 case, while maintaining approximate balancing property.
However, \citet{du2018algorithmic} do not consider any noise in the update, 
 hence their results  can not provide further theoretical insights on understanding the role of noise when applying first order algorithms to nonconvex problems.

In this paper, we are interested in studying the algorithmic behavior of first order algorithms in the presence of noise.
Specifically, we study a perturbed form of gradient descent (Perturbed GD) applied to the matrix factorization problem  \eqref{mf_obj}, which 
 injects independent noise to iterates, and then evaluates gradient at the perturbed iterates.
Note that our algorithm is different from SGD in terms of the noise.
For our algorithm, we inject independent noise to the iterates $(X_t, Y_t)'s$ and use the gradient evaluated at the perturbed iterates.
The noise of SGD, in contrast, comes from the training sample. As a consequence, the noise of SGD has very complex dependence on the iterate, which is difficult to analyze. See more detailed discussions in Sections \ref{sec_discussion}.

We further analyze the convergence properties of our Perturbed GD algorithm for the rank-1 case.
At the early stage, noise helps the algorithm to escape from regions with undesired landscape, including the strict saddle point.
After entering the region with benign landscape, Perturbed GD behaves similarly to gradient descent, until the loss is sufficiently small.   
At the early stage, noise provides additional explorations that help the algorithm to escape from the strict saddle point.
Then at the  later stage, the noise dominates the update of Perturbed GD, and gradually rescales the iterates to a balanced solution that is uniquely determined by the injected noise.
Specifically, the ratio of the norm $\norm{x_t}_2/\norm{y_t}_2$ is completely determined by the ratio of the variance of noise injected to $(x_t, y_t)$.
To the best of our knowledge, this is the first theoretical result towards understanding the implicit bias of noise in nonconvex optimization problems.
Our analysis  reveals an interesting characterization of the local landscape around global minima, which relates to the sharp/flat minima in deep neural networks \citep{keskar2016large}, and we will further discuss these connections in detail in Section \ref{sec_discussion}.
We believe that investigating the implicit bias of the noise in nonconvex matrix factorization can serve as a fundamental building block for studying stochastic optimization for more sophiscated nonconvex problems, including training over-parameterized neural networks.

\noindent{\textbf{Notations}}: $\mathbf{1}$ Given a matrix $A$, $\tr(A)$ denotes the trace of $A.$ For matrices $A,B\in \RR^{n\times m},$  we use $\inner{A}{B}$ to denote the Frobenius inner product, i.e., $\inner{A}{B}=\tr(A^\top B).$ $\norm{A}_{\rm{F}}=\sqrt{\inner{A}{A}}$ denotes the Frobenius norm of $A.$ $I_d\in \RR^{d\times d}$ denotes the identity matrix.

\section {Model and Algorithm}\label{sec_model}
We first describe the nonconvex matrix factorization problem, and then present the perturbed gradient descent algorithm (Perturbed GD).
For simplicity, we primarily focus on the  rank-1 matrix factorization problem. Extensions to the rank-r case are provided in Section \ref{sec_extend}.
\subsection{ Rank-1 Matrix Factorization}
We consider the following nonconvex optimization problem: 
\begin{align}\label{mat_fct_ncvx}
\min_{x\in\RR^{d_1}, y\in\RR^{d_2}}\cF(x,y)=\frac{1}{2}\|xy^\top-M\|^2_{\rm{F}},
\end{align}
 where $M \in \RR^{d_1\times d_2}$ is a rank-1 matrix and can be factorized as follows:
$\displaystyle
M=u_*v_*^\top,
$
where $u_*\in\RR^{d_1}$, $v_*\in\RR^{d_2}$. Without loss of generality, we assume $\norm{u_*}_2=\norm{v_*}_2=1.$

The optimization landscape of \eqref{mat_fct_ncvx} has been well studied in the previous literature \citep{ge2016matrix, ge2017no, chi2019nonconvex, li2019symmetry}.  
 Because of the bilinear form in  $\cF,$ there exist infinitely many global minima to \eqref{mat_fct_ncvx}, including highly unbalanced ones, i.e.,   $xy^\top=M$ with $\norm{x}_2\gg\norm{y}_2$ or $\norm{x}_2\gg\norm{y}_2$ (see Definition \ref{def:balance}).
In Section \ref{main}, we will show that such unbalancedness essentially implies global minima with a large condition number.
 To address this issue, \citet{tu2015low,ge2017no} propose a regularizer of the form $(\norm{x}_2^2-\norm{y}_2^2)^2$ to balance $\norm{x}_2$ and $\norm{y}_2$.
 Recently,  \citet{du2018algorithmic} show that even without explicit regularization, gradient descent  with small random initialization converges to  balanced solutions with constant probability.
 Yet, all of the previous results assume noiseless updates. 
The algorithmic behavior of first order algorithms with noisy updates remains unclear for the nonconvex matrix factorization problem.

\subsection{Perturbed Gradient Descent}
To study the effect of noise, we consider a perturbed gradient descent algorithm (Perturbed GD).
At the $t$-th iteration, we first perturb the iterate $(x_t,y_t)$ with independent Gaussian noise $\xi_{1,t}\sim N(0,\sigma_1^2I_{d_1})$ and $\xi_{2,t}\sim N(0,\sigma_2^2I_{d_2}),$ respectively.
We then update $(x_t, y_t)$ with the gradient evaluated at the perturbed iterates. 
The detail of the Perturbed GD algorithm is summarized in Algorithm \ref{alg:Perturbed GD}.
\begin{algorithm}[H]
    \caption{Perturbed Gradient Descent for Rank-1 Matrix Factorization.}
    \label{alg:Perturbed GD}
    \begin{algorithmic}
    	\STATE{\textbf{Input}: step size $\eta$, noise level $\sigma_1, \sigma_2$, matrix $M \in \RR^{d_1 \times d_2}$, number of iterations $T$.}
	\STATE{\textbf{Initialize}: initialize $(x_0, y_0)$ arbitrarily.}
	\FOR{$t = 0 \ldots T-1$}
	\STATE{Sample $\xi_{1,t}  \sim N(0, \sigma_1^2 I_{d_1})$ and $\xi_{2,t} \sim N(0, \sigma_2^2 I_{d_2})$.}
	\STATE{$\tilde{x}_t = x_t + \xi_{1,t}, ~ \tilde{y}_t = y_t + \xi_{2,t}.$}
	\STATE{$x_{t+1} = x_t - \eta (\tilde{x}_t \tilde{y}_t^\top - M ) \tilde{y}_t.$}
	\STATE{$y_{t+1} = y_t - \eta (\tilde{y}_t \tilde{x}_t^\top - M^\top ) \tilde{x}_t.$}
	\ENDFOR
    \end{algorithmic}
\end{algorithm}

\textbf{Smoothing Effect}.
 Using Perturbed GD to solve \eqref{mat_fct_ncvx}  can also be viewed as solving the following stochastic optimization problem:
\begin{align}\label{mat_fact_eq}
\min_{x\in\RR^{d_1}, y\in\RR^{d_2}}\tilde\cF(x,y)=\EE_{\xi_{1},\xi_{2}}\cF(x+\xi_1,y+\xi_2),
\end{align}
where $\xi_{1}\sim N(0,\sigma_1^2 I_{d_1})$ and $\xi_{2}\sim N(0,\sigma_2^2 I_{d_2}).$ 
Throughout this paper, we will refer to problem \eqref{mat_fact_eq} as the smoothed problem.
The expectation in  \eqref{mat_fact_eq} can be viewed as convoluting the objective function with a Gaussian kernel. 
In Section  \ref{main}, we show that this convolution effectively  smooths out  unbalanced optima  and yields a benign landscape.

\begin{remark}
The use of random noise to convolute with the objective function is also known as randomized smoothing, which is first proposed in \citet{duchi2012randomized}.
\citet{zhou2019towards,jin2017escape} further exploit this effect to explain the importance of noise in helping first order algorithms to escape from strict saddle points and spurious local optima.
\end{remark}

\section{Main Results}\label{main}
We study the algorithmic behavior of our proposed perturbed gradient descent (Perturbed GD) algorithm. 
 We primarily focus on the rank-1 nonconvex matrix factorization.
 We first characterize the landscape of the original problem \eqref{mat_fct_ncvx},
 and show that noise effectively smooths the original problem, yielding a smoothed problem \eqref{mat_fact_eq} with benign landscape.
 We then provide a non-asymptotic convergence analysis of the Perturbed GD, and demonstrate the implicit bias induced by the noise.
 In particular, we consider the case where noise is balanced, see more details in Theorem \ref{thm_main}.
 Due to space limit, we defer all the proofs to the appendix. 

We analyze the landscape of the original problem \eqref{mat_fct_ncvx} and the smoothed problem \eqref{mat_fact_eq}. 
Note that due to the bilinear form in $\cF(x,y)$, we have for  $\alpha \neq 0,$ $\displaystyle\cF\left(\alpha x,\alpha^{-1} y\right)=\cF( x,y).$
The scaling invariance nature of \eqref{mat_fct_ncvx} results in undesired landscape properties, and makes the 
analysis of first order algorithms particularly difficult.
To facilitate further discussions, below we characterize the landscape of  \eqref{mat_fct_ncvx}.

\begin{lemma}[\bf Landscape Analysis]\label{lem_landscape}
The gradients of $\cF$ with respect to $x$ and $y$ take the form:
\begin{align*}
\nabla_{x}\cF(x,y)=(xy^\top-M)y,~
\nabla_{y}\cF(x,y)=(xy^\top-M)^{\top}x.
\end{align*}
Then $\cF$ has two types of stationary points: (i) For any $\alpha\neq 0,$  $\left(\alpha u_*, \alpha^{-1} v_*\right)$ is a global optimum; (ii) For any $x\in \RR^{d_1},y\in\RR^{d_2}$ such that $x^\top u_*=y^\top v_*=0,$   $(x,0)$ and $(0,y)$ are strict saddle.
\end{lemma}

The scaling-invariance leads to infinitely many global optima for \eqref{mat_fct_ncvx}, each taking the form $(\alpha u_*, \alpha^{-1} v_*)$. However,  Lemma \ref{lem_condition_number} shows that  different values of $\alpha$ lead to significantly different local landscape around the global minima.

\begin{lemma}\label{lem_condition_number}
The condition number  of  the Hessian matrix of $\cF$ at the global optimum $(\alpha u_*,\alpha^{-1} v_*)$ is $$\kappa\left(\nabla^2 \cF\left(\alpha u_*,\alpha^{-1} v_*\right)\right)=\max\left\{\alpha^4,\alpha^{-4}\right\}+1.$$
\end{lemma}
\begin{figure*}
\centering
\includegraphics[width=\linewidth]{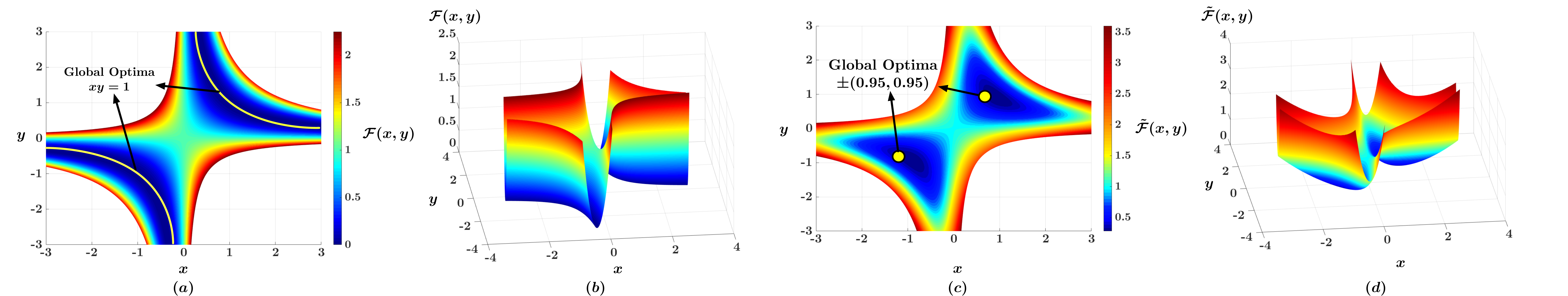}
\vspace{-0.2in}
\caption{ The visualization of objective functions $\cF(x,y)=(1-xy)^2$ and  $\tilde\cF(x,y)$ with $x,y\in \RR$ and $\sigma_1^2=\sigma_2^2=0.0975.$  For $\cF(x,y),$ any $(x,y)$ that satisfies $xy=1$ is a global minimum (as shown in (a)). $\tilde\cF(x,y)$ only has global minima close to $\pm(1,1)$ (as shown in (c)).}\label{fig:landscape}
\vspace{-0.1in}
\end{figure*}
From Lemma \ref{lem_condition_number} we know  that when $|\alpha|$ is extremely large or small, 
 the global minimum $(\alpha u_*, \alpha^{-1} v_*)$ is ill-conditioned (large condition number).
The condition number also relates to the sharp/flat minima in neural networks, which we will discuss in detail in Section \ref{sec_discussion}.
Our previous discussion implies that any global optimum $(x_*, y_*)$  to \eqref{mat_fct_ncvx} will be ill-conditioned if the norm ratio $\norm{x_*}_2/\norm{y_*}_2$ is close to zero or infinity.
To facilitate further discussions, we define the balancedness property as follows.

\begin{definition}[$\gamma$-balancedness]\label{def:balance}
We say that $(x,y)\in(\RR^{d_1},\RR^{d_2})$ is $\gamma-$balanced for some positive number $\gamma$ if  $\norm{x}_2=\gamma\norm{y}_2.$
Informally, we say $(x,y)$ is unbalanced if $\gamma$ is close to zero or infinity, and $(x,y)$ is balanced if $\gamma$ is close to 1.\end{definition}

Our next lemma shows that, in the presence of noise, the smoothed problem  \eqref{mat_fact_eq}  has only balanced optima, and the balancedness is completely determined by the ratio of the variance of the noise injected to the iterates.
\begin{lemma}\label{lem_convolutional}
	We say the noise in the Perturbed GD is $\gamma-$balanced if $\EE\left[\norm{\xi_1}_2^2\right]=\gamma^2\EE\left[\norm{\xi_2}_2^2\right]$. 
	For $\gamma = 1$, we say the noise is balanced.
	For any noise ratio $\gamma>0,$ if we take 
	$\sigma^2=\EE\left[\norm{\xi_1}_2^2\right]=\gamma^2\EE\left[\norm{\xi_2}_2^2\right]\leq \min\{\gamma,1\},$
	\eqref{mat_fact_eq} has two global optima $(\tilde u_*(\gamma),\tilde v_*(\gamma))=\pm\sqrt{\gamma-\sigma^2}\left(u_*,\gamma^{-1}v_*\right)$ and one saddle point $(0,0)$ with  $\lambda_{\min}(\nabla^2 \tilde F(0,0))<0$.
\end{lemma}

 We see that  when the noise is $\gamma$-balanced and the noise level is sufficiently small, 
 the global minima of 
 the smoothed problem \eqref{mat_fact_eq} can be arbitrarily close to the $\gamma$-balanced global minima of the original rank-1 matrix factorization problem  \eqref{mat_fct_ncvx}. 
 Compared with Lemma \ref{lem_landscape}, Lemma \ref{lem_convolutional} shows that noise effectively smooths the landscape of the original problem  \eqref{mat_fct_ncvx},  the unbalanced optima (with ill-conditioned Hessian) to \eqref{mat_fct_ncvx} are no longer optima after convolution. 
See Figure \ref{fig:landscape} for  detailed illustration of the landscape for the original problem  \eqref{mat_fct_ncvx}  and the smoothed  problem  \eqref{mat_fact_eq}.

Moreover, one can verify that the saddle point $(0,0)$ of problem \eqref{mat_fact_eq}  enjoys the strict saddle property, i.e., the smallest eigenvalue of Hessian at $(0,0)$ is negative. By the stable manifold theory proposed in \citet{lee2019first},  Perturbed GD with an infinitesimal step size and vanishing noise avoids all the strict saddle points and converges to the global optima, asymptotically. However, this asymptotic result does not provide any finite time guarantee. 

We next present the non-asymptotic convergence analysis of our proposed Perturbed GD algorithm.
Our analysis shows that Perturbed GD has implicit bias  determined by the injected noise. 
Specifically, we consider balanced noise, i.e., 
$\EE\left[\norm{\xi_1}_2^2\right]=\EE\left[\norm{\xi_2}_2^2\right]$,
and show that 
Perturbed GD converges to the balanced $\pm(u_*,v_*)$ in polynomial time in Theorem \ref{thm_main}.

\begin{theorem}[\bf Convergence Analysis]\label{thm_main}
Suppose  $x_0\in \RR^{d_1},$  $y_0\in \RR^{d_2}.$ For any $\epsilon>0$ and for any $\delta\in(0,1),$  we take 
$\sigma^2=\EE\left[\norm{\xi_1}_2^2\right]=\EE\left[\norm{\xi_2}_2^2\right]=\rm{poly}(\epsilon,(\log(1/\delta))^{-1}),$
$$\eta={\rm{poly}}( \sigma,\epsilon, (d_1+d_2)^{-1},(\log{(d_1+d_2)})^{-1},(\log(1/\delta))^{-1}).$$


With probability at least $1-\delta,$ we have $\norm{x_t-u_*}_2\leq \epsilon$ and $\norm{y_t-v_*}_2\leq \epsilon.$ for all $t_1\leq t\leq  T_1= O\left({1}/{\eta^2}\right),$ where
$t_1=O\left(\eta^{-1} \sigma^{-2}\log(1/\eta )\log(1/\delta)\right).$
\end{theorem}
 Theorem \ref{thm_main} differs from the convergence analysis of GD in \citet{du2018algorithmic} in the following aspects: (i) Perturbed GD converges regardless of initializations, while GD requires a random initialization near $(0,0)$; 
 (ii) Perturbed GD guarantees convergence with high probability, while GD converges with only constant probability over randomness of initializations;
  (iii) Perturbed GD converges to the balanced solutions ($\norm{x}_2/\norm{y}_2=1$) while GD only maintains approximate balancedness, i.e.,  $c_0\leq\norm{x}_2/\norm{y}_2\leq C_0$ for some absolute constants $c_0, C_0>0.$ 
 
We provide a proof sketch that contains the essential ingredients of characterizing the convergence, since the proof of Theorem \ref{thm_main} is very technical and highly involved. 
See more details and the proof of technical lemmas in Appendix \ref{pf_2}.

\begin{proof}[Proof Sketch]
The convergence of Perturbed GD consists of  three phases: 
{\bf Phase I}: Regardless of initialization, within polynomial time the noise encourages Perturbed GD to escape from region with undesired landscape (e.g., strict saddle points)  and enter the region with benign landscape.
{\bf Phase II}: Perturbed GD drives the loss to zero and approaches the set of global minima $\left\{(x,y)\big|xy^\top=M\right\}.$
{\bf Phase III}: After the loss is sufficiently small, 
the injected noise dominates the update, which helps the Perturbed GD to balance $x$ and $y,$ and converge to a balanced optimum.

Before we proceed with our proof, 
we first define some notations. Let $U$ and $V$ denote the linear span of $u_*$ and $v_*$, respectively, i.e.,
$
\cU=\{\alpha_1 u_*:\alpha_1\in\RR\},~~\cV=\{\alpha_1 v_*:\alpha_1\in\RR\}.
$
The corresponding orthogonal complement of $\cU$ (or $\cV$) in $\RR^{d_1}$ (or $\RR^{d_2}$) is denoted as $\cU^\perp$ (or $\cV^\perp$).
Our  analysis considers the convergence of Perturbed GD in $(\cU,\cV)$ and $(\cU^\perp,\cV^\perp),$ respectively. Specifically, we take the following orthogonal decomposition of $x_t$ and $y_t$:
\begin{align*}
x_t&=u_*^\top x_t  u_*+(x_t-u_*^\top x_t   u_*),\\y_t&=v_*^\top y_t   v_*+(y_t-v_*^\top y_t   v_*).
\end{align*}
One can check that $(x_t-u_*^\top x_t   u_*)\in\cU^\perp$ and $(y_t-v_*^\top y_t   v_*)\in \cV^\perp,$ respectively. 

\noindent $\bullet$ {\bf Phase I}: Regardless of initializations,
 the following lemma shows that $(x_t-u_*^\top x_t   u_*)$ and $(y_t-v_*^\top y_t   v_*)$ vanish after polynomial time.
\begin{lemma}\label{lem_b_converge}
Suppose  $\norm{x_t}^2_2+\norm{y_t}^2_2\leq 2/\sigma^2$ holds for all $t>0.$ For any $\delta\in(0,1)$  we take 
$$\eta\leq \eta_2=C_4{\sigma^8}\left(\log((d_1+d_2)/\delta)\log(1/\delta)\right)^{-1},$$  where $C_4$ is some positive constant. Then with probability at least $1-\delta,$ we have 
\begin{align*}
\norm{x_t-u_*^\top x_t   u_*}_2^2&\leq  2\eta C_2\sigma^{-2},\\
\norm{y_t-v_*^\top y_t   v_*}_2^2&\leq  2\eta C_2\sigma^{-2}
\end{align*}
for any $\tau_1\leq t\leq T_1= O(1/\eta^2),$ where $\tau_1=O(\eta^{-1}\sigma^{-2}\log(1/\eta)\log(1/\delta))$ and  $C_2=(\sigma^2+1/\sigma^2)(2/\sigma^4+6/d_1+6/d_2+6\sigma^4).$
\end{lemma}
Lemma \ref{lem_b_converge} shows that with an arbitrary initialization, the projection of $x_t$ (and $y_t$) onto $\cU^\perp$ (and $\cV^\perp$) vanishes. 
Thus, we only need to characterize the algorithmic behavior of Perturbed GD in subspace $(\cU,\cV).$  
However, as the projection of $x_t$ (and $y_t$) onto $\cU^\perp$ (and $\cV^\perp$) vanishes,
Perturbed GD can possibly approach the region with undesired landscape, e.g., the small neighborhood around the saddle point.
The next lemma shows that the Perturbed GD will escape such regions within polynomial time.

\begin{lemma}\label{lem_escape}
Suppose $\norm{x_t-u_*^\top x_t   u_*}_2^2\leq  2\eta C_2\sigma^{-2}$ and $\norm{y_t-v_*^\top y_t   v_*}_2^2\leq  2\eta C_2\sigma^{-2}$ hold for all $t>0.$ For any $\delta\in(0,1),$ we take $\eta\leq\eta_3=C_3\sigma^{12}\left(\log((d_1+d_2)/\delta)\log(1/\delta)\right)^{-1},$ where $C_3$ is some positive constant.
Then with probability at least $1-\delta,$  we have 
\begin{align}\label{eq_escape}
x_{t}^\top u_* v_*^\top y_{t}=x_{t}^\top M y_{t}\geq 1/4,
\end{align}
for all $\tau_2\leq t\leq T_1=O\left({1}/{\eta^2}\right),$ where $\tau_2=O\left(\eta^{-1}\sigma^{-2}\log(1/\eta)\log(1/\delta)\right).$
\end{lemma}
Since the saddle point $(x,y)=(0,0)$ satisfies $x^\top M y = 0,$ Lemmas \ref{lem_b_converge} and \ref{lem_escape} together imply that regardless of initializations, Perturbed GD is bounded away from the saddle point  after polynomial time.
 The algorithm then enters Phase II and approaches   the set of global optima $\left\{(x,y)\big|xy^\top=M\right\}.$

 
 \noindent $\bullet$  {\bf Phase II}:
 In this phase, $(x_t,y_t)$ is still  away from the region $\left\{(x,y)\big|xy^\top=M\right\}.$ Thus,  $\nabla\cF(x_{t},y_{t})$ dominates the update of Perturbed GD.  Perturbed GD behaves similarly to gradient descent while driving the loss to zero. The next lemma formally characterizes this behavior, showing that $xy^
 \top$ converges to $M.$ 
 \begin{lemma}\label{lem_loss}
 	Suppose $x_{t}^\top u_* v_*^\top y_{t}\geq\frac{1}{4}$ holds for all $t>0.$ For any $\epsilon>0$ and for any $\delta\in(0,1)$, we choose $\sigma\leq\sigma'_1=C_4\sqrt{\epsilon}$ and take 
 	$\eta\leq \eta_4={C}_5{\sigma^6}\left(\log((d_1+d_2)/\delta)\log(1/\delta)\right)^{-1},$  where $C_4,C_5$  are some positive constants. Then with probability at least $1-\delta,$ we have $$\norm{x_ty_t^\top-M}_{\rm{F}}\leq \epsilon,$$
 	for all $\tau_3\leq t\leq  T_1= O\left({1}/{\eta^2}\right),$ where
 	$\tau_3=O\left(\eta^{-1}\log\frac{1}{\sigma}\log(1/\delta)\right).$
 	\end{lemma}
 
 \noindent $\bullet$  {\bf Phase III}: After Phase II, Perturbed GD enters the region where $x_ty_t^\top\approx M.$ Thus, the noise will dominate the update of  the Perturb GD.
Recall in Lemma \ref{lem_condition_number}, we show that the Hessian of unbalanced optima has a large condition number, hence a small perturbation will significantly change the gradient of the objective.
This implies that the unbalanced optima would be  unstable against the noise, and Perturbed GD will escape from such optima and continue iterating towards the balanced optima. 
 The next lemma shows that $x_t $  and $y_t $ converge to $u_*, v_*,$ respectively.
\begin{lemma}\label{lem_convergence}
For $\forall\epsilon>0$, suppose $\norm{x_ty_t^\top-M}_{\rm{F}}\leq \epsilon$ holds for all $t>0.$ For any $\delta\in(0,1)$, we choose $\sigma\leq\sigma'_2={C}_6\left(\log(1/\delta)\right)^{-1/3}$ and take 
$\eta\leq \eta_5=C_7\sigma^{10}\epsilon,$  where $C_6, C_7$ are some positive constants. Then with probability at least $1-\delta,$ we have 
\begin{align*}
\norm{x_t -u_*}_2\leq \epsilon,~\norm{y_t-v_*}_2\leq \epsilon,
\end{align*}
for all $\tau_3\leq t\leq  T_1= O\left({1}/{\eta^2}\right),$ where
$\tau_4=O\left(\eta^{-1}\sigma^{-2}\log\eta^{-1}\log(1/\delta)\right).$
\end{lemma}

With all the lemmas in place, we take $\sigma\leq \min\{\sigma'_1,\sigma'_2\},$ $\eta\leq \min\{\eta_1,\eta_2,\eta_3,\eta_4,\eta_5\}$ and $t_1=\tau_1+\tau_2+\tau_3+\tau_4$, then the claim of Theorem \ref{thm_main} follows immediately.
\end{proof}

\vspace{-0.12in}
\section{Extension to Rank-r Matrix Factorization}\label{sec_extend}
\vspace{-0.1in}
We extend our results to the rank-$r$ matrix factorization, which solves the following problem:
\begin{align}\label{rankr_mf}
\min_{X \in \RR^{d_1 \times r},Y \in \RR^{d_2 \times r}} \cF(X,Y) =  \frac{1}{2}\norm{XY^\top - M}_{\rm{F}}^2,
\end{align}  
where $M\in \RR^{d_1\times d_2}$ is a rank-$r$ matrix. 
 Let $M=A\Sigma B^\top$ be the SVD of $M.$ 
Let $U_*=A\Sigma^{\frac{1}{2}}$ and $V_*=B\Sigma^{\frac{1}{2}},$ then $(U_*, V_*)$ is a global minimum for problem \eqref{rankr_mf}.
Similar to the rank-1 case, using Perturbed GD to solve problem \eqref{rankr_mf} can be viewed as solving the following smoothed problem:
\begin{align}\label{rankr_eq}
\min_{X,Y}\tilde\cF(X,Y)=\EE_{\xi_{1},\xi_{2}}\cF(X+\xi_1,Y+\xi_2),
\end{align}
where $\xi_{1}\in \RR^{d_1\times r},$ $\xi_{2}\in \RR^{d_2\times r}$ have i.i.d. elements drawn from $N(0,\sigma_1^2)$ and   $N(0,\sigma_2^2),$ respectively.


 The next theorem shows that the noise in Perturbed GD effectively addresses the scaling invariance issue of \eqref{rankr_mf}.
  The smoothed problem \eqref{rankr_eq} only has balanced global minima.
\begin{theorem}\label{thm_rankr}
Let $\sigma_{\min}(M)$ be the smallest singular value of $M$. Suppose $$\EE\left[\norm{\xi_1}_{\rm{F}}^2\right]=\gamma^2\EE\left[\norm{\xi_2}_{\rm{F}}^2\right]=r\gamma^2\sigma^2,$$ and  $\gamma\sigma^2<\sigma_{\min}(M).$ Then  for $\forall(U,V)\in (\RR^{d_1},\RR^{d_2})$ such that $\nabla\tilde\cF(U,V)=0,$ we have
$U^\top U=\gamma^2 V^\top V.$
Moreover,  denote $(\tilde U,\tilde V)=\left(\sqrt{\gamma} A(\Sigma-\gamma\sigma^2I_r)^{\frac{1}{2}},\gamma^{-1/2}B(\Sigma-\gamma\sigma^2I_r)^{\frac{1}{2}}\right),$ then the set  $\{(\tilde U,\tilde V)R \big|R \in \RR^{r\times r}, RR^\top=R^\top R=I_r \}$ contains all the global optima. All other stationary points are strict saddles, i.e., $\lambda_{\min}\left(\nabla^2 \tilde\cF(U,V)\right)<0.$
\end{theorem}

Theorem \ref{thm_rankr} shows that  when noise is balanced, \eqref{rankr_eq}  only has balanced global optima and strict saddle points. 
We can invoke \citet{lee2019first} again and show Perturbed GD with an infinitesimal step size and vanishing noise converges to the balanced global optima, asymptotically. 

Note that compared to the rank-1 case, in addition to scaling invariance, the objective is also rotation invariant, i.e., $\tilde\cF(X,Y)= \tilde\cF(XR,YR),$ where $R\in \RR^{r\times r}$ is an orthogonal matrix. 
Thus, we can only recover $U_*$ and $V_*$ up to a rotation factor.
To establish the non-asymptotic convergence result, the optimization error should be measured by the following metric that is rotation invariant:
\begin{align*}
\mathrm{dist}_\cR(D_1,D_2)=\min_{R\in \RR^{r\times r}:RR^\top=I_r}\norm{D_1-D_2R}_{\rm{F}}.
\end{align*}
However, it is more challenging and involved to handle the complex nature of the distance $\mathrm{dist}_\cR(\cdot, \cdot)$. 
As our results for the rank-1 case already provide insights on understanding the implicit bias of noise, we leave the  non-asymptotic analysis of rank-r matrix factorization for future investigation.


\section{Numerical Experiments}\label{sec_numerical}

We present numerical results to support our theoretical findings. 
We compare our Perturbed GD algorithm with gradient descent (GD), and demonstrate that Perturbed GD with $\gamma-$balanced noise converges to $\gamma-$balanced optima, while the optima obtained by GD are highly sensitive to initialization and step size.
We also show that the phase transition between Phase II and Phase III is not an artifact of the proof, and faithfully captures the true algorithmic behavior of Perturbed GD.

\noindent\textbf{Rank-1 Matrix Factorization}.
We first consider the rank-1 matrix factorization problem. 
Without loss of generality, the matrix $M$ to be factorized is given by $M = u_* v_*^\top$, where $u_* = (1, 0, \ldots, 0)\in \RR^{d_1}$ and $v_* = (1, 0, \ldots, 0)\in \RR^{d_2}$, with $d_1 = 20$ and $d_2 = 30$. 
We initialize iterates $(x_0, y_0)$ with $x_0 \sim N(0,\sigma_x^2 I_{d_1}),\  y_0 \sim N(0,\sigma_y^2 I_{d_2})$. 
For all experiments, we use $\gamma$-balanced noise in Perturbed GD. 
Specifically, we choose $\xi_{1,t} \sim N(0, \sigma_1^2 I_{d_1}), \xi_{2,t} \sim N(0, \sigma_2^2 I_{d_2})$. 

\noindent\textbf{(1) Balanced Noise}.
We  consider the case of balanced noise ($\gamma =1$), where we take $\sigma_1 = \sqrt{1.5} \times 0.05$ and $\sigma_2 = 0.05$.
One can verify that $\gamma^2 = {d_1 \sigma_1^2}/{(d_2 \sigma_2^2)}  = 1$.

We first use balanced step size to compare with GD studied in  \citet{du2018algorithmic}.
Specifically, we set $\eta_x = \eta_y = 10^{-2}$ for both Perturbed GD and GD.
We further consider two initialization schemes:
(i) small initializations (which is also adopted in \citet{du2018algorithmic}): $\sigma_x = \sigma_y = 10^{-2}$; (ii) large initializations: $\sigma_x = \sigma_y = 10^{-1}$. 
Fig. \ref{mainresult}.(a, b, d, e)  summarize the results of $100$ repeated experiments in a box-plot. 
As can be seen, regardless of initializations, Perturbed GD always converges to the balanced optima (Fig. \ref{mainresult}.(a, b)).
In contrast, GD only converges to approximately balanced optima for small initializations (Fig. \ref{mainresult}.(d)), and the large initialization yields a large variance in terms of the balancedness of the obtained solution (Fig. \ref{mainresult}.(e)). 
  \begin{figure*}[t]
\centering
\includegraphics[width=0.9\linewidth]{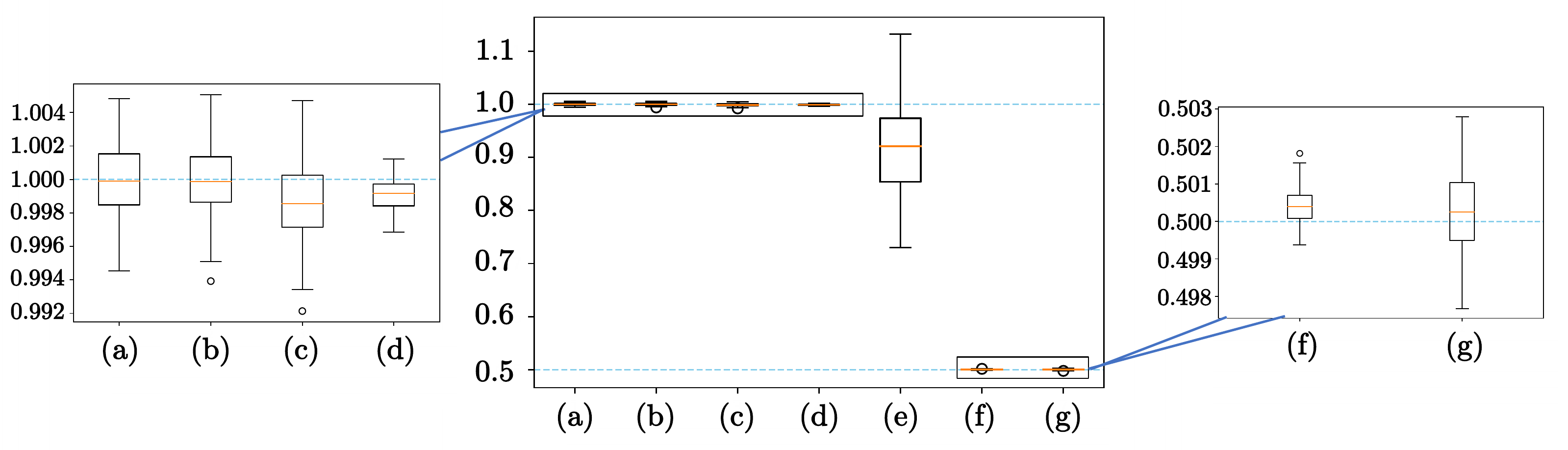}
    \caption{Perturbed GD with balanced noise (a, b and c), Perturbed GD with unbalanced noise (g) and GD (d,e and f) for the rank-1 matrix factorization problem. (a) and (d) use small initializations ($\sigma_x = \sigma_y = 10^{-2}$) and balanced step size ($\eta_x = \eta_y = 10^{-2}$). (b) and (e) use large initializations ($\sigma_x = \sigma_y = 10^{-1}$) and balanced step size. (c) and (f) use small initializations and unbalanced step size ($\eta_x = 0.5 \eta_y  =5 \times 10^{-3}$).}\label{mainresult}
\end{figure*}

We also use unbalanced step size (different step sizes for updating $x$ and $y$) and compare the convergence properties of Perturbed GD and GD. 
Specifically, we set $\eta_x = 0.5 \eta_y  =5 \times 10^{-3}$ for both Perturbed GD and GD.
We adopt a small initialization scheme, with $\sigma_x  = \sigma_y = 10^{-2}$.
Fig. \ref{mainresult}.(c, f)  summarize the results of $100$ repeated simulations in a box-plot. 
As can be seen, even with small initializations, GD with unbalanced step size converges to the approximately $\sqrt{0.5}$-balanced optima, instead of the $1$-balanced optima (Fig. \ref{mainresult}.(f)). 
In contrast, Perturbed GD is able to converge to the balanced optima with unbalanced step size (Fig. \ref{mainresult}.(c)). 

Our results suggest that the noise is the most important factor in determining the balancedness of the solutions obtained by Perturbed GD. 


\noindent\textbf{(2) Unbalanced Noise}.
We run Perturbed GD with unbalanced noise.
We take $\sigma_1 = \sqrt{0.75} \times 0.05$ and $\sigma_2 = 0.05$ with $\gamma^2 ={d_1 \sigma_1^2}/
{(d_2 \sigma_2^2)}= 0.5$. 
We use a  small initialization: $\sigma_x = \sigma_y = 10^{-2}$, and balanced step size: $\eta_x = \eta_y = 10^{-2}$.
Fig. \ref{mainresult}.(g)  summarizes the results of $100$ repeated simulations in a box-plot.
As can be seen,  for $\gamma \neq 1$, the Perturbed GD converges to the $\gamma$-balanced optima.

\noindent \textbf{Rank-10 Matrix Factorization}. We then consider rank-10 nonconvex matrix factorization problem. 
The matrix $M$ to be factorized is given by $M = U_* V_*^\top$, where 
$
   U_* =\begin{pmatrix}
I_{10}, 0
\end{pmatrix}_{d_1\times10}^\top,
 V_* =\begin{pmatrix}
I_{10}, 0
\end{pmatrix}_{d_2\times10}^\top,
$
with $d_1 = 20$ and $d_2 = 30$. 
We initialize iterates $(X_0, Y_0)$ with all entries ${X_0}^{(i,j)}$'s and ${Y_0}^{(i,j)}$'s independently sampled from $N(0,\sigma_x^2)$ and $N(0,\sigma_y^2)$, respectively.   
For all experiments, we use $\gamma$-balanced noise in Perturbed GD. 
Specifically, we choose  $\xi_{1,t}$ and $\xi_{2,t}$  with i.i.d. elements drawn from $N(0,\sigma_1^2)$ and $N(0,\sigma_2^2)$ respectively.
We repeat a similar set of experiments as in rank-1 case, and summarize the results in Fig. \ref{mainresultrk-10}.
For each of the experiments, we use the same set of $(\eta_x, \eta_y, \sigma_x, \sigma_y)$ as their counterpart in the rank-1 case.
As can be seen, Perturbed GD always converges to the $\gamma-$balanced optima, regardless of initializations. Our experiments suggest that, for  the  rank-r matrix factorization problem, the noise still determines the balancedness of the optima obtained by Perturbed GD. 
\begin{figure*}[t]
\centering
\includegraphics[width=0.9\linewidth]{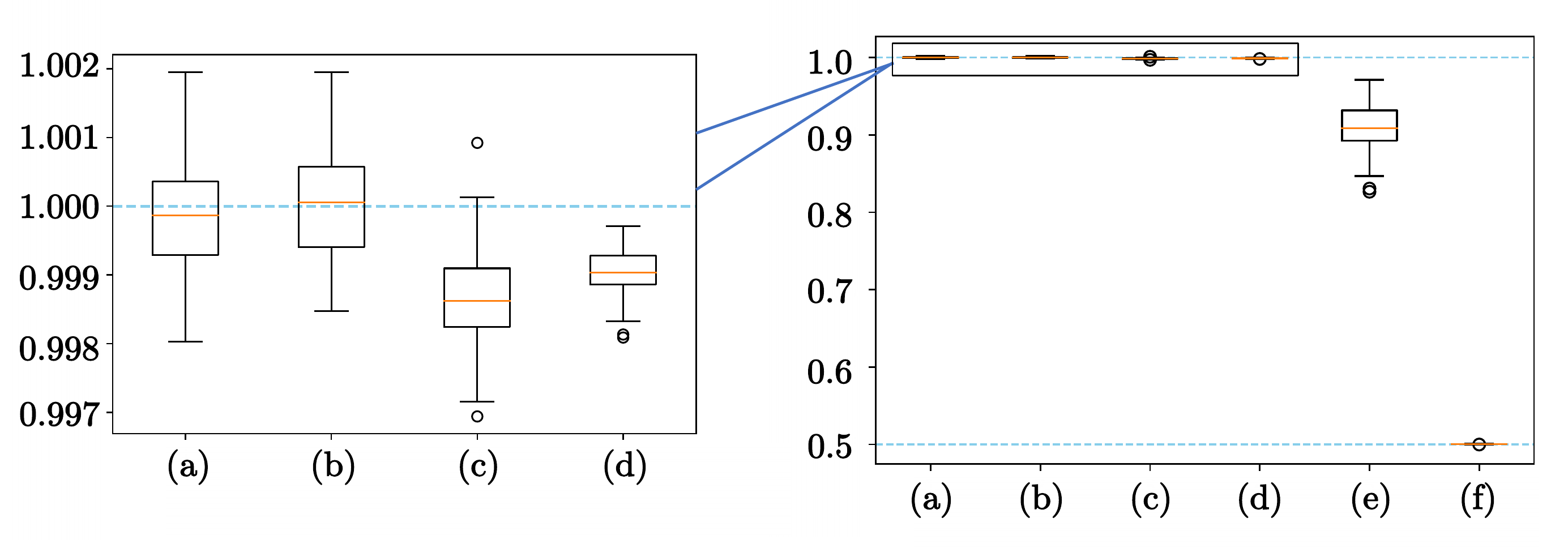}
    \caption{Perturbed GD with balanced noise (a, b and c) and GD (d,e and f) for rank-10 matrix factorization problem. (a) and (d) use small initializations ($\sigma_x = \sigma_y = 10^{-2}$) and balanced step size ($\eta_x = \eta_y = 10^{-2}$). (b) and (e) use large initializations ($\sigma_x = \sigma_y = 10^{-1}$) and balanced step size. (c) and (f) use small initializations and unbalanced step size ($\eta_x = 0.5 \eta_y  =5 \times 10^{-3}$).}\label{mainresultrk-10}
\end{figure*}

\noindent \textbf{Phase Transition.} We further demonstrate the transition between Phase II and Phase III in the Perturbed GD algorithm. Specifically, we consider 2-dimensional problem $f(x,y) = (1-xy)^2$ with balanced optima $\pm(1,1)$. 
We set $\sigma_1 = \sigma_2 = 0.05$, initialize $(x_0 ,y_0) = (3,5)$, and use balanced step size $\eta_x = \eta_y = 0.01$.
We repeat the experiments $50$ times and summarize the result of one realization in Fig. \ref{2phase}, as the convergence properties of Perturbed GD are highly consistent across different realizations.
We also use exponential moving average to smooth the loss trajectory to better illustrate the overall progress of the objective in Phase III.

As can be seen, in around the first $30$ to $40$ iterations,
Perturbed GD and GD behave similarly. 
 Both  Perturbed GD and GD iterate towards the set of global optima $\{(x,y)\big| xy=1\}$, while driving the loss to zero, and the squared norm ratio $x_t^2/y_t^2$ in Perturbed GD decreases from $0.36$ to around $0.004$.
 \begin{figure}[t]
\includegraphics[width=0.9\linewidth]{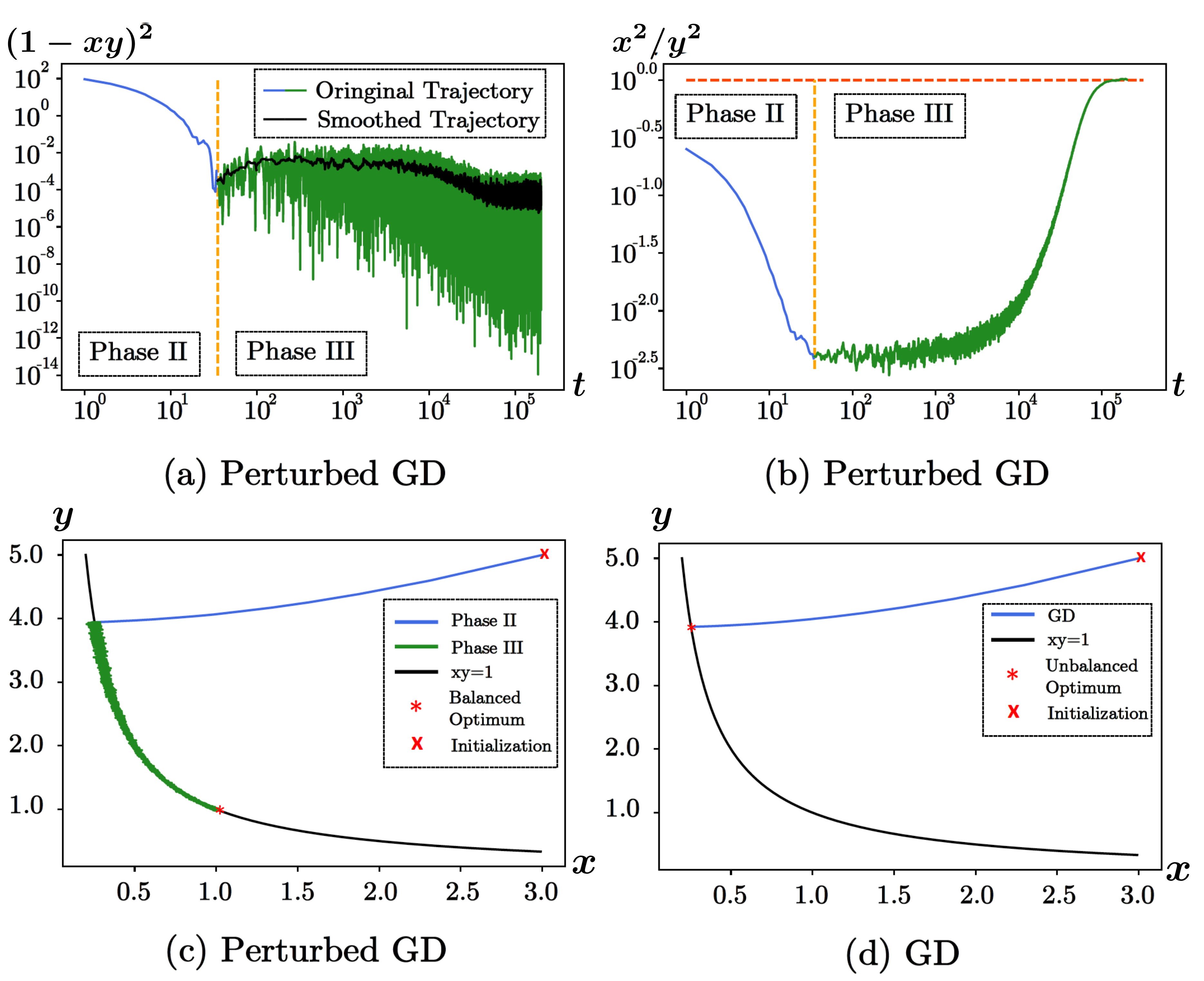}
  \caption{Algorithmic behaviors of Perturbed GD and GD. For Perturbed GD, phase transition happens around the first $30\!\sim\!40$ iterations, as shown in (a,b,c). GD does not show phase transitions.  }\label{2phase}
\end{figure}
After that, GD converges to the unbalanced optimum. Since the loss is sufficiently small, the noise dominates the update of Perturbed GD.
Then the squared norm ratio $x_t^2/y_t^2$ gradually increases from $0.004$ to $1$.
Perturbed GD  iterates towards the balanced optimum while staying  close to  global optima. 

The phase transition phenomenon can also been observed for higher dimensional problems. As shown in Figure \ref{fig:addit} for $d=4$,  the loss greatly decreases to  and stay around $10^{-4}$   in the first $2\times10^3$ iterations, and then the squared norm ratio ${\norm{x}_2^2}/{\norm{y}_2^2}$ gradually increases from $0.5$ to $1$. This implies the transition between Phase II and Phase III, that is Perturbed GD first approaches the set of global minima and then converges to the balanced optimum. 
\begin{figure}[t]\label{fig:addit}
\includegraphics[width=0.48\linewidth]{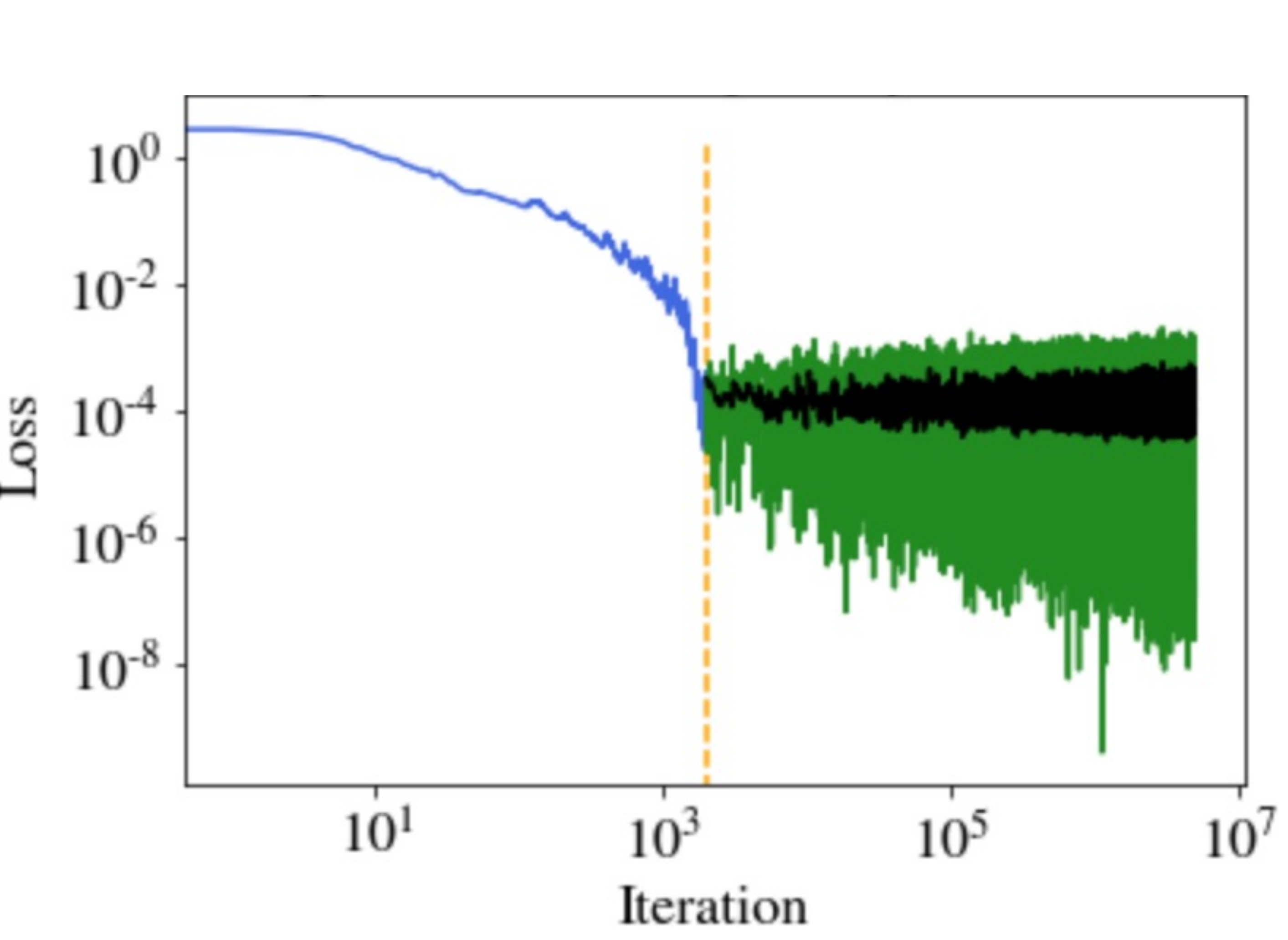}
\includegraphics[width=0.465\linewidth]{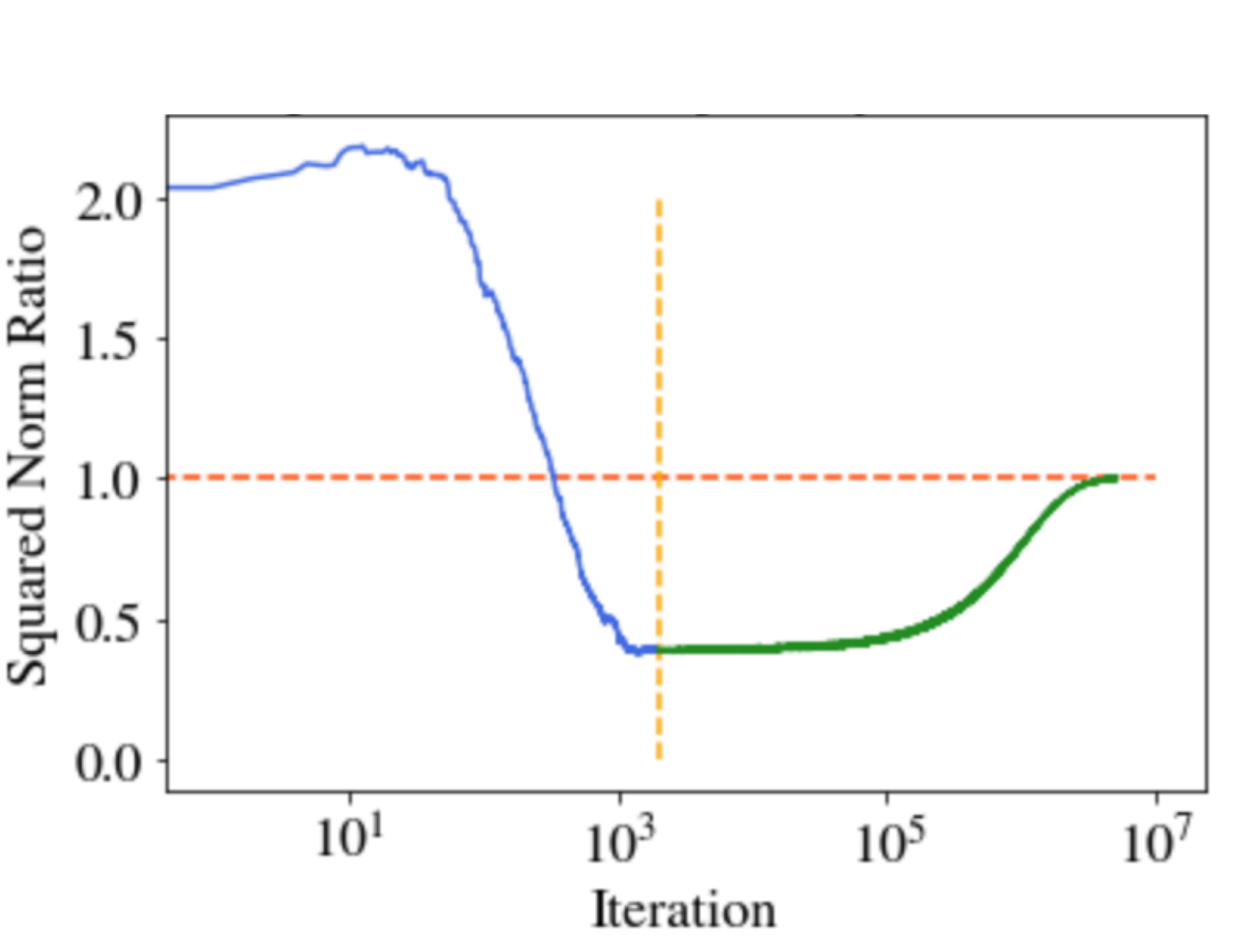}

  \caption{Algorithmic behaviors of Perturbed GD and GD for $d=4$. For Perturbed GD, phase transition happens around the first $2\times10^3$ iterations.  }
\end{figure}


\vspace{-0.1in}
\section{Discussions}\label{sec_discussion}
\vspace{-0.1in}
\noindent {\bf Connections to SGD.}
This paper studies the implicit bias of the noise in  nonconvex optimization. 
Direct analysis on SGD is beyond current technical limit due to complex dependencies.
Specifically,  the noise in SGD comes from random sampling of the training data, and heavily depends on the iterate. 
This induces a complex dependency between iterate and the noise, and  makes it difficult to characterize the distribution of the noise.

Our Perturbed GD can be viewed as a close variant of SGD.
Moreover, the noise in Perturbed GD follows Gaussian distribution and is independent of the iterates. Hence, the analysis of Perturbed GD, though still highly non-trivial, is now technically manageable.

\noindent {\bf Biased Gradient Estimator.} Different from SGD,  Perturbed GD implements a biased gradient estimator, i.e., $ \EE_{\xi_{1},\xi_{2}}\nabla\cF(x+\xi_1,y+\xi_2)\neq\nabla\cF(x,y).$ Such biased gradient also appears in training deep neural networks combined with computational heuristics.
 Specifically, \cite{luo2018towards} show that with batch normalization,  the gradient estimator in SGD is  also biased with respect to the original loss. The similarity between this biased gradient and our perturbed gradient is worth future investigation.


\noindent {\bf Extension to Other Types of Noise.} Our work considers Gaussian noise, but can be extended to analyzing other types of noise. For example, we can show  that anisotropic noise will have different smoothing
effects along different directions. The implicit bias will thus depend on the covariance of noise in addition to the noise level. For another example, heavy tailed distribution of noise will affect the probability of the convergence of Perturbed GD. It may be difficult to achieve high probability convergence as we have shown for light tailed distributions.

\noindent {\bf Sharp/Flat Minima and Phase Transition.} 
Lemma \ref{lem_condition_number} shows that the Hessian matrix of an unbalanced global optimum is ill-conditioned, and the landscape around such an optimum is sharp in some directions and flat in others. 
For nonconvex matrix factorization, all the global optima are connected and form a path.
The landscape around the path forms a valley, which is narrow around unbalanced optima and wide around the balanced ones (See Figure \ref{fig:landscape2}).
Our three-phase convergence analysis shows that the Perturbed GD first falls into the valley, and then traverses within the valley until it finds the balanced optima.
\begin{figure}[t]
\centering
\includegraphics[width=0.95\linewidth]{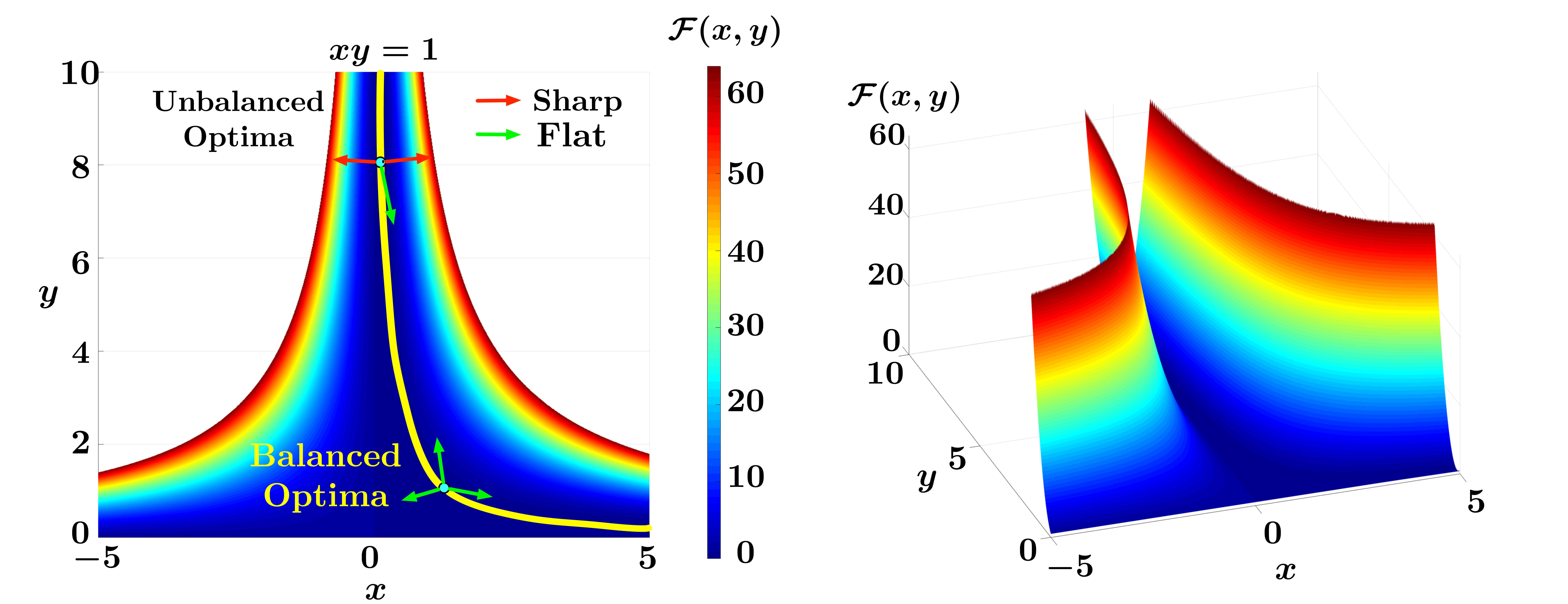}
\caption{ The visualization of objective   $\cF(x,y)=(1-xy)^2.$  All the global optima are connected and form a path. The landscape around the path forms a valley.  Around unbalanced optima, the landscape  is sharp in some directions and flat in others.  Around balanced optima, the landscape only contains flat directions.}
\label{fig:landscape2}
\end{figure}

As we have mentioned earlier, people have shown that the local optima of deep neural networks are also  connected. Thus, the phase transition also provides a new perspective to explain the plateau of training curves after learning rate decay in training neural networks. 
Our analysis suggests that after adjusting the learning rate, the algorithm enters a new phase, where the noise slowly re-adjusts the landscape.
At the beginning of this phase, the loss decreases rapidly due to the reduced noise level.
By the end of this phase, the algorithm falls into a region with benign landscape that is suitable for further decreasing the step size.

\noindent {\bf Related Literature.} 
Implicit bias of noise has also been studied in \citet{haochen2020shape,blanc2020implicit}. However, they consider perturbing labels while our work considers perturbing parameters.
In a broader sense, our work is also related to \citet{li2019towards} which show that the noise scale of SGD may change the learning order of patterns. However, they do not study the implicit bias of noise towards certain optima.

\bibliography{ref}
\bibliographystyle{ims}


\section{Preliminaries}
We first introduce some important notions and results which can be used in the following proof. 

Assuming  $\{u_*,\tilde{u}_1,\dots,\tilde{u}_{d_1-1}\}$ and $\{v_*,\tilde{v}_1,\dots,\tilde{v}_{d_2-1}\}$ are two sets of standard orthogonal basis of $\RR^{d_1}$ and $\RR^{d_2}$ respectively, we can then rewrite $\forall x\in\RR^{d_1}$ and $\forall y\in\RR^{d_2}$ as 
\begin{align*}
x\triangleq \alpha_1u_*+\sum_{i=1}^{d
_1-1} \beta_{1}^{(i)}\tilde{u}_i,\\
y\triangleq \alpha_2v_*+\sum_{j=1}^{d
_2-1} \beta_2^{(j)}\tilde{v}_j,
\end{align*}
where $\alpha_1=x^\top u_*,$ $\alpha_2=y^\top v_*,$ $\beta_{1}^{(i)}=x^\top \tilde{u}_i$ and $\beta_{2}^{(j)}=y^\top \tilde{v}_j, ~\forall 0\leq i\leq d_1-1, 0\leq j\leq d_2-1.$
For simplicity, we denote $\beta_k=(\beta_k^{(1)},...,\beta_k^{(d_k-1)})^\top$ where $k=1,2.$

With the notions above, we can rewrite the Perturbed GD update as
\begin{align*}
\alpha_{1,t+1}=\alpha_{1,t}-\eta<\nabla_{x}\cF(x_{t}+\xi_{1,t},y_{t}+\xi_{2,t}),u_*>,\\\alpha_{2,t+1}=\alpha_{2,t}-\eta<\nabla_{y}\cF(x_{t}+\xi_{1,t},y_{t}+\xi_{2,t}),v_*>.
\end{align*}
Note that the optimal solutions to \eqref{mat_fct_ncvx} satisfy $x^\top My=1$. Thus, in our proof, we need to characterize the update of $x^\top My,$ which can be re-expressed as
\begin{align}
x_{t+1}^\top My
_{t+1}&=\alpha_{1,t+1}\alpha_{2,t+1}\nonumber\\&=(\alpha_{1,t}-\eta<\nabla_{x}\cF(x_{t}+\xi_{1,t},y_{t}+\xi_{2,t}),u_*>)(\alpha_{2,t}-\eta<\nabla_{y}\cF(x_{t}+\xi_{1,t},y_{t}+\xi_{2,t}),v_*>)\nonumber\\&=\alpha_{1,t}\alpha_{2,t}-\eta\left(\alpha_{2,t}<\nabla_{x}\cF(x_{t}+\xi_{1,t},y_{t}+\xi_{2,t}),u_*>+\alpha_{1,t}<\nabla_{y}\cF(x_{t}+\xi_{1,t},y_{t}+\xi_{2,t}),v_*>\right)\nonumber\\&\quad\quad\quad\quad\quad\quad\quad\quad+\eta^2<\nabla_{x}\cF(x_{t}+\xi_{1,t},y_{t}+\xi_{2,t}),u_*><\nabla_{y}\cF(x_{t}+\xi_{1,t},y_{t}+\xi_{2,t}),v_*>.
\end{align}
For simplicity, we denote
\begin{align*}
&A_t\triangleq \alpha_{2,t}<\nabla_{x}\cF(x_{t}+\xi_{1,t},y_{t}+\xi_{2,t}),u_*>+\alpha_{1,t}<\nabla_{y}\cF(x_{t}+\xi_{1,t},y_{t}+\xi_{2,t}),v_*>,\\&B_t\triangleq <\nabla_{x}\cF(x_{t}+\xi_{1,t},y_{t}+\xi_{2,t}),u_*><\nabla_{y}\cF(x_{t}+\xi_{1,t},y_{t}+\xi_{2,t}),v_*>.
\end{align*}
Then the update of $x^\top My=\alpha_1 \alpha_2$ can be expressed in a more compact way as follows.
\begin{align}\label{*1}
\alpha_{1,t+1}\alpha_{2,t+1}=\alpha_{1,t}\alpha_{2,t}-\eta A_t+\eta^2B_t.
\end{align}
Similarly, the update of $(x^\top My-1)^2$ can be re-expressed as
\begin{align}\label{*2}
(x_{t+1}^\top My_{t+1}-1)^2&=(\alpha_{1,t+1}\alpha_{2,t+1}-1)^2
\nonumber\\&=(\alpha_{1,t}\alpha_{2,t}-1-\eta A_t+\eta^2B_t)^2
\nonumber\\&=(\alpha_{1,t}\alpha_{2,t}-1)^2+\eta^2 A_t^2+\eta^4 B_t^2\nonumber\\&~~~~~~~~~~~~-2\eta A_t(\alpha_{1,t}\alpha_{2,t}-1)-2\eta^3A_tB_t+2\eta^2B_t(\alpha_{1,t}\alpha_{2,t}-1).
\end{align}
Furthermore, since  the balanced optima satisfy $x^\top u_*=y^\top v_*,$ we further explicitly write down the update of $\left((x^\top u_*)^2-(y^\top v_*)^2\right)^2$  as follows.
\begin{align}\label{*3}
\left((x_{t+1}^\top u_*)^2-(y_{t+1}^\top v_*)^2\right)^2&=(\alpha_{1,t+1}^2-\alpha_{2,t+1}^2)^2\nonumber\\
&=\left(\alpha_{1,t}^2-\alpha_{2,t}^2\right)^2+4\eta^2 D_t^2+\eta^4 F_t^2\nonumber\\&~~~~~~~~~~~~-4\eta D_t\left(\alpha_{1,t}^2-\alpha_{2,t}^2\right)-4\eta^3D_t F_t+2\eta^2F_t\left(\alpha_{1,t}^2-\alpha_{2,t}^2\right),
\end{align}
where $D_t$ and $F_t$ is defined as
\begin{align*}
&D_t\triangleq \alpha_{1,t}<\nabla_{x}\cF(x_{t}+\xi_{1,t},y_{t}+\xi_{2,t}),u_*>-\alpha_{2,t}<\nabla_{y}\cF(x_{t}+\xi_{1,t},y_{t}+\xi_{2,t}),v_*>,\\&F_t\triangleq <\nabla_{x}\cF(x_{t}+\xi_{1,t},y_{t}+\xi_{2,t}),u_*>^2-<\nabla_{y}\cF(x_{t}+\xi_{1,t},y_{t}+\xi_{2,t}),v_*>^2.
\end{align*}

Next we are going to calculate $<\nabla_{x}\cF(x_{t}+\xi_{1,t},y_{t}+\xi_{2,t}),u_*>$, the gradient projection along the direction of the optimum $u_*$.
\begin{align}\label{*4}
&<\nabla_{x}\cF(x_{t}+\xi_{1,t},y_{t}+\xi_{2,t}),u_*>\nonumber\\&=u_*^\top\nabla_{x}\cF(x_{t}+\xi_{1,t},y_{t}+\xi_{2,t})
\nonumber\\&=u_*^\top\left((x_{t}+\xi_{1,t})(y_{t}+\xi_{2,t})^\top-u_*v_*^\top\right)(y_{t}+\xi_{2,t})
\nonumber\\&=\alpha_{1,t}(\alpha_{2,t}^2+\norm{\beta_{2,t}}_2^2)-\alpha_{2,t}+2\alpha_{1,t}y_{t}^\top \xi_{2,t}-v_*^\top \xi_{2,t}+u_*^\top \xi_{1,t}(\alpha_{2,t}^2+\norm{\beta_{2,t}}_2^2)\nonumber\\&~~~~~~~~~~~~+2u_*^\top\xi_{1,t}y_{t}^\top \xi_{2,t}+\alpha_{1,t}\norm{\xi_{2,t}}^2+u_*^\top \xi_{1,t}\norm{\xi_{2,t}}^2
\nonumber\\&\triangleq \alpha_{1,t}(\alpha_{2,t}^2+\norm{\beta_{2,t}}_2^2)-\alpha_{2,t}+g_x,
\end{align}
where $g_x=2\alpha_{1,t}y_{t}^\top \xi_{2,t}-v_*^\top \xi_{2,t}+u_*^\top \xi_{1,t}(\alpha_{2,t}^2+\norm{\beta_{2,t}}_2^2)+2u_*^\top\xi_{1,t}y_{t}^\top \xi_{2,t}+\alpha_{1,t}\norm{\xi_{2,t}}^2+u_*^\top \xi_{1,t}\norm{\xi_{2,t}}^2.$

Similarly, we have 
\begin{align}\label{*5}
&<\nabla_{y}\cF(x_{t}+\xi_{1,t},y_{t}+\xi_{2,t}),v_*>
\nonumber\\&=\alpha_{2,t}(\alpha_{1,t}^2+\norm{\beta_{1,t}}_2^2)-\alpha_{1,t}+2\alpha_{2,t}x_{t}^\top \xi_{1,t}-u_*^\top \xi_{1,t}+v_*^\top \xi_{2,t}(\alpha_{1,t}^2+\norm{\beta_{1,t}}_2^2)\nonumber\\&~~~~~~~~~~~~+2v_*^\top\xi_{2,t}x_{t}^\top \xi_{1,t}+\alpha_{2,t}\norm{\xi_{1,t}}^2+v_*^\top \xi_{2,t}\norm{\xi_{1,t}}^2
\nonumber\\&\triangleq \alpha_{2,t}(\alpha_{1,t}^2+\norm{\beta_{1,t}}_2^2)-\alpha_{1,t}+g_y,
\end{align}
where $g_y=\alpha_{1,t}+2\alpha_{2,t}x_{t}^\top \xi_{1,t}-u_*^\top \xi_{1,t}+v_*^\top \xi_{2,t}(\alpha_{1,t}^2+\norm{\beta_{1,t}}_2^2)+2v_*^\top\xi_{2,t}x_{t}^\top \xi_{1,t}+\alpha_{2,t}\norm{\xi_{1,t}}^2+v_*^\top \xi_{2,t}\norm{\xi_{1,t}}^2.$
\section{Proof of Lemma \ref{lem_landscape}, Lemma \ref{lem_condition_number} , and Lemma \ref{lem_convolutional}}\label{pf_1}

\subsection{Proof of Lemma \ref{lem_landscape}}
\begin{proof}
By setting the gradient of $ \cF$ to zero we get
\begin{align}
    &\norm{x}_2^2y=M^\top x,\label{zero_grad_1}\\
    &\norm{y}_2^2x=My.\label{zero_grad_2}
\end{align}
	Recall that  $M=u_* v_*^\top$ and \begin{align*}
x\triangleq \alpha_1u_*+\sum_{i=1}^{d
_1-1} \beta_{1}^{(i)}\tilde{u}_i,\\
y\triangleq \alpha_2v_*+\sum_{j=1}^{d
_2-1} \beta_2^{(j)}\tilde{v}_j.
\end{align*}
Substitute $x$ and $y$ in \eqref{zero_grad_1} and \eqref{zero_grad_2} by their expansion, we then have 
\begin{align*}
(\alpha_1^2+\sum_{i=1}^{d_1-1} (\beta_{1}^{(i)})^2)\alpha_2&=\alpha_1,\\
(\alpha_2^2+\sum_{i=1}^{d_2-1} (\beta_{2}^{(i)})^2)\alpha_1&=\alpha_2,\\
(\alpha_1^2+\sum_{i=1}^{d_1-1} (\beta_{1}^{(i)})^2)\beta_{2}^{(j)}&=0, \forall j=1,...,d_2-1,\\
(\alpha_2^2+\sum_{i=1}^{d_2-1} (\beta_{2}^{(i)})^2)\beta_{1}^{(j)}&=0, \forall j=1,...,d_1-1.
\end{align*}
The above equalities yield the following two types of stationary points.
\begin{itemize}
\item $(\alpha_1^2+\sum_{i=1}^{d_1-1} (\beta_{1}^{(i)})^2)(\alpha_2^2+\sum_{i=1}^{d_2-1} (\beta_{2}^{(i)})^2)\neq 0,\beta_1=0,\beta_2=0.$ This leads to $\alpha_1\alpha_2=1.$  Thus, $xy^\top=\alpha_1\alpha_2 u_*v_*^\top=M,$ and $\cF(\alpha_1u_*,\alpha_2v_*)=0.$ Then we have global optima $\left(\alpha u_*, \frac{1}{\alpha} v_*\right)$ for $\alpha\neq 0.$
\item Either $(\alpha_1^2+\sum_{i=1}^{d_1-1} (\beta_{1}^{(i)})^2)=0$ and $\alpha_2=0,$ or $(\alpha_2^2+\sum_{i=1}^{d_2-1} (\beta_{2}^{(i)})^2)=0$ and $\alpha_1=0.$ We next show that stationary points satisfy these conditions are strict saddle points.  We only consider the first case, and the second case can be proved following similar lines.  We first calculate the Hessian matrix as follows.
\begin{align*}
    \nabla^2 \cF(x, y)=\begin{pmatrix}
\norm{y}_2^2I_{d_1} & 2xy^\top-M \\
2yx^\top-M^\top & \norm{x}_2^2I_{d_2}
\end{pmatrix}.
\end{align*}
At $x=0$, $y=\sum_{j=1}^{d_2-1} \beta_2^{(j)}\tilde{v}_j,$
\begin{align*}
    \nabla^2 \cF(x, y)=\begin{pmatrix}
\sum_{i=1}^{d_2-1} (\beta_{2}^{(i)})^2I_{d_1} & -M \\
-M^\top & 0
\end{pmatrix}.
\end{align*}
For any $a\in \RR^{d_1},b\in \RR^{d_2},$
\begin{align*}
\begin{pmatrix}
a^\top, b^\top
\end{pmatrix}\nabla^2 \cF(x, y)\begin{pmatrix}
a  \\
b 
\end{pmatrix}=\sum_{i=1}^{d_2-1} (\beta_{2}^{(i)})^2\norm{a}_2^2-2a^\top M b
\end{align*}
For $(a, b)=(\tilde{u}_i, \tilde{v}_j)$, this quantity is positive. For  $(a, b)=(u_*, \sum_{i=1}^{d_2-1} (\beta_{2}^{(i)})^2v_*),$ this quantity is negative. Thus, $x=0$, $y=\sum_{j=1}^{d_2-1} \beta_2^{(j)}\tilde{v}_j$satisfies strict saddle property. We conclude that for any $x\in \RR^{d_1},y\in\RR^{d_2}$ such that $x^\top u_*=y^\top v_*=0,$ we have strict saddle points $(x,0)$ and $(0,y).$
 \end{itemize}
\end{proof}
\subsection{Proof of Lemma \ref{lem_condition_number}}
At $x=\alpha u_*, y=\frac{1}{\alpha}v_*,$
\begin{align*}
    \nabla^2 \cF(\alpha u_*, \frac{1}{\alpha}v_*)=\begin{pmatrix}
\frac{1}{\alpha^2}I_{d_1} &M \\
M^\top & \alpha^2I_{d_2}
\end{pmatrix}.
\end{align*}
One can verify that   $\nabla^2 \cF\left(\alpha u_*,\frac{1}{\alpha} v_*\right)$ has eigenvalues $\alpha^2+\frac{1}{\alpha^2},\alpha^2, \frac{1}{\alpha^2}.$ The largest eigenvalue is $\lambda_1=\alpha^2+\frac{1}{\alpha^2}$ and the smallest eigenvalue is $\lambda_{d_1+d_2}=\min\{\alpha^2,\frac{1}{\alpha^2}\}.$ Thus, the condition number can be easily calculated as follows.
$$\kappa\left(\nabla^2 \cF\left(\alpha u_*,\frac{1}{\alpha} v_*\right)\right)=\max\{\alpha^4,\frac{1}{\alpha^4}\}+1.$$
\subsection{Proof of Lemma \ref{lem_convolutional}}
\begin{proof}
Recall that  $M=u_* v_*^\top$ and \begin{align*}
x\triangleq \alpha_1u_*+\sum_{i=1}^{d
_1-1} \beta_{1}^{(i)}\tilde{u}_i,\\
y\triangleq \alpha_2v_*+\sum_{j=1}^{d
_2-1} \beta_2^{(j)}\tilde{v}_j.
\end{align*}

By setting the gradient of $\tilde \cF$ to zero we get
\begin{align*}
    &(\norm{x}_2^2+d_1 \sigma_1^2)y=M^\top x,\\
    &(\norm{y}_2^2+d_2 \sigma_2^2)x=My.
\end{align*}
From the equations above, we can verify that $(x,y)=(0,0)$ is a stationary point. Furthermore, left multiplying the equations above by $\tilde{v}_j^\top$ and $\tilde{u}_i^\top$ respectively, we will get that $\beta_{1}^{(i)}$'s and $\beta_{2}^{(j)}$'s are all zeros. Similarly, by left multiplying the equations above by $v_*^\top$ and $u_*^\top$ respectively, we will get\begin{align*}
    &(\alpha_1^2+d_1 \sigma_1^2)\alpha_2^2=\alpha_2\alpha_1,\\
    &(\alpha_2^2+d_2 \sigma_2^2)\alpha_1^2=\alpha_1\alpha_2.
\end{align*}
Then, with some algebraic manipulations, we get
\begin{align*}
    &\alpha_1^2=\sqrt{\frac{d_1 \sigma_1^2}{d_2 \sigma_2^2}} - d_1 \sigma_1^2,\\
    &\alpha_2^2=\sqrt{\frac{d_2 \sigma_2^2}{d_1 \sigma_1^2}} - d_2 \sigma_2^2.
\end{align*}
Specifically, when $d_1 \sigma_1^2=\gamma^2 d_2 \sigma_2^2\leq\gamma$, we have
\begin{align*}
    &\alpha_1=\gamma \alpha_2= \pm \sqrt{\gamma- \gamma^2 d_2 \sigma_2^2}.
\end{align*}
Next, we are going to show that $(0, 0)$ is a strict saddle point and $(x_*, y_*)\triangleq\pm(\alpha_1u_*, \alpha_2v_*)$ are global optima. We first calculate the Hessian matrix as follows.
\begin{align*}
    \nabla^2 \tilde\cF(x, y)=\begin{pmatrix}
\left(\norm{y}_2^2+d_2\sigma_2^2\right)I_{d_1} & 2xy^\top-M \\
2yx^\top-M^\top & \left(\norm{x}_2^2+d_1\sigma_1^2\right)I_{d_2}
\end{pmatrix}.
\end{align*}
Since the injected noise is small, we have $\alpha_1\alpha_2=1-\sqrt{d_1\sigma_1^2d_2\sigma_2^2}>0$. For any $a \in \RR^{d_1}$ and $b \in \RR^{d_2}$, we have\begin{align*}
    &\begin{pmatrix}
a^\top, b^\top
\end{pmatrix}\nabla^2 \tilde\cF(x_*, y_*)\begin{pmatrix}
a  \\
b 
\end{pmatrix}\\
=&(\alpha_2^2+d_2\sigma_2^2)\norm{a}_2^2+(\alpha_1^2+d_1\sigma_1^2)\norm{b}_2^2+2(2\alpha_1\alpha_2-1)(a^\top u_*) (b^\top v_*)\\
\geq&(\alpha_2^2+d_2\sigma_2^2)\norm{a}_2^2+(\alpha_1^2+d_1\sigma_1^2)\norm{b}_2^2-2|2(1-\sqrt{d_1\sigma_1^2d_2\sigma_2^2})-1|\ \norm{a}_2\ \norm{b}_2\\
>&\sqrt{\frac{d_2 \sigma_2^2}{d_1 \sigma_1^2}}\norm{a}_2^2+\sqrt{\frac{d_1 \sigma_1^2}{d_2 \sigma_2^2}}\norm{b}_2^2-2\norm{a}_2\ \norm{b}_2,
\end{align*}where the last inequality comes from the fact that $d_1\sigma_1^2$ and $d_2\sigma_2^2$ should be small enough such that $$0<1-\sqrt{d_1\sigma_1^2d_2\sigma_2^2}<1.$$Note that $$\sqrt{\frac{d_2 \sigma_2^2}{d_1 \sigma_1^2}}\norm{a}_2^2+\sqrt{\frac{d_1 \sigma_1^2}{d_2 \sigma_2^2}}\norm{b}_2^2-2\norm{a}_2\ \norm{b}_2=\left((\frac{d_2 \sigma_2^2}{d_1 \sigma_1^2})^{\frac{1}{4}}\norm{a}_2-(\frac{d_1 \sigma_1^2}{d_2 \sigma_2^2})^{\frac{1}{4}}\norm{b}_2\right)^2.$$
Thus, as long as ${d_1\sigma_1^2\ d_2\sigma_2^2}<1$, we have $\nabla^2 \cF_{reg}(x_*, y_*)$ is positive definite (PD). Thus $(x_*, y_*)$ is global minimum and so is $(-x_*,- y_*)$.

Similarly, for any $a \in \RR^{d_1}$ and $b \in \RR^{d_2}$, we have
\begin{align*}
    &\begin{pmatrix}
a^\top, b^\top
\end{pmatrix}\nabla^2 \cF_{reg}(0, 0)\begin{pmatrix}
a  \\
b 
\end{pmatrix}
=d_2\sigma_2^2\norm{a}_2^2+d_1\sigma_1^2\norm{b}_2^2-2(a^\top u_*) (b^\top v_*).
\end{align*}
It is easy to show that for $(a, b)=(u_*, v_*)$, the quantity is negative when noise is small enough. But for $(a, b)=(\tilde{u}_i, \tilde{v}_j)$, this quantity is positive. Thus, (0,0) is a strict saddle point. Here, we prove Lemma  \ref{lem_convolutional}. 
\end{proof}

\section{Proof of Theorem \ref{thm_main}}\label{pf_2}
\subsection{Boundedness of Trajectory}
We first show that the solution trajectory of Perturbed GD is bounded with high probability, which is a sufficient condition for our following convergence analysis.
\begin{lemma}[Boundedness of Trajectories]\label{lem_bounded}
	Given $x_0\in\RR^{d_1},$ $y_0\in\RR^{d_2},$ we choose $\sigma_1, \sigma_2>0$ such that $\sigma^2=\EE\left[\norm{\xi_1}_2^2\right]=\EE\left[\norm{\xi_2}_2^2\right]$ and $\norm{x_0}^2_2+\norm{y_0}^2_2\leq 1/\sigma^2.$      For any $\delta\in(0,1)$, we  take  $$\eta\leq \eta_1= C_1{\sigma^6}(\log((d_1+d_2)/\delta)\log(1/\delta))^{-1},$$ for some positive constant $C_1.$ Then with probability at least $1-\delta,$ we have $\norm{x_t}^2_2+\norm{y_t}^2_2\leq 2/\sigma^2$ for any $t\leq T_1=O(1/\eta^2).$
\end{lemma}
\begin{proof}
We first define the event where the injected noise for both $x$ and $y$ is bounded for the first $t$ iterations.
\begin{align}
&\cA_t=\left\{\big|{\xi^{(i)}_{1,\tau}}\big|,\big|{\xi^{(j)}_{2,\tau}}\big|\leq \sigma\left(\sqrt{2\log((d_1+d_2)\eta^{-2})} +\sqrt{\log(1/\delta)}\right), \forall\tau\leq t, i=1,...,d_1,j=1,..,d_2 \right\}.
\end{align}
By the concentration result of the maximum of Gaussian distribution,  we have $\PP(\cA_{1/\eta^2})\geq 1-\delta.$
Moreover, we use  $\cH_t$ to denote the event where  the first $t$ iterates $\{(x_\tau,y_\tau)\}_{\tau\leq t}$ is bounded, i.e.,
$
\cH_t=\left\{\norm{x_\tau}^2_2+\norm{y_\tau}^2_2\leq \frac{2}{\sigma_2},\forall\tau\leq t\right\}.
$
Let $\cF_t=\sigma\{(x_\tau,y_\tau),\tau\leq t\}$ denote the $\sigma-$algebra generated by that past $t$ iterations. 

Under the event $\cH_t$ and $\cA_t,$ we have the following inequality on the conditional expectation of  $\norm{x_{t+1}}_2^2+\norm{y_{t+1}}_2^2.$
\begin{align}
&\EE\left[(\norm{x_{t+1}}_2^2+\norm{y_{t+1}}_2^2)\mathds{1}_{\cH_t\cap\cA_t}\big|\cF_t\right]\nonumber\\=&\left\{(1-2\eta\sigma^2)\left(\norm{x_{t}}_2^2+\norm{y_{t}}_2^2\right)-4\eta\left(\norm{x_{t}}_2^2\norm{y_{t}}_2^2-x_t^\top M y_t\right)\right\}\mathds{1}_{\cH_t\cap\cA_t}\nonumber\\
&+\eta^2 \EE_{\xi_{1,t},\xi_{2,t}}\left[\norm{\nabla_{x}\cF(x_{t}+\xi_{1,t},y_{t}+\xi_{2,t})}_2^2+\norm{\nabla_{y}\cF(x_{t}+\xi_{1,t},y_{t}+\xi_{2,t})}_2^2\right]\mathds{1}_{\cH_t\cap\cA_t}\nonumber\\
\leq&\left\{(1-2\eta\sigma^2)\left(\norm{x_{t}}_2^2+\norm{y_{t}}_2^2\right)-4\eta\left(\norm{x_{t}}_2^2\norm{y_{t}}_2^2-\norm{x_{t}}_2\norm{y_{t}}_2\right)\right\}\mathds{1}_{\cH_t\cap\cA_t}\nonumber\\
&+\eta^2 \EE_{\xi_{1,t},\xi_{2,t}}\left[ \big|\big|\norm{y_t+\xi_{2,t}}_2^2(x_t+\xi_{1,t})-M(y_t+\xi_{2,t})\big|\big|_2^2\right]\mathds{1}_{\cH_t\cap\cA_t}\nonumber\\
&+\eta^2 \EE_{\xi_{1,t},\xi_{2,t}}\left[ \big|\big|\norm{x_t+\xi_{1,t}}_2^2(y_t+\xi_{2,t})-M^\top(x_t+\xi_{1,t})\big|\big|_2^2\right]\mathds{1}_{\cH_t\cap\cA_t}\nonumber\\
\leq&\left\{(1-2\eta\sigma^2)\left(\norm{x_{t}}_2^2+\norm{y_{t}}_2^2\right)-4\eta\left(\norm{x_{t}}_2^2\norm{y_{t}}_2^2-\norm{x_{t}}_2\norm{y_{t}}_2\right)\right\}\mathds{1}_{\cH_t\cap\cA_t}\nonumber\\
&+2\eta^2 \EE_{\xi_{1,t},\xi_{2,t}}\left[ \norm{y_t+\xi_{2,t}}_2^4\norm{x_t+\xi_{1,t}}_2^2+\norm{M(y_t+\xi_{2,t})}_2^2\right]\mathds{1}_{\cH_t\cap\cA_t}\nonumber\\
&+2\eta^2 \EE_{\xi_{1,t},\xi_{2,t}}\left[ \norm{x_t+\xi_{1,t}}_2^4\norm{y_t+\xi_{2,t}}_2^2+\norm{M^\top(x_t+\xi_{1,t})}_2^2\right]\mathds{1}_{\cH_t\cap\cA_t}\nonumber\\
\leq&\left\{(1-2\eta\sigma^2)\left(\norm{x_{t}}_2^2+\norm{y_{t}}_2^2\right)+\eta+\eta^2 C_2\right\}\mathds{1}_{\cH_t\cap\cA_t},\nonumber
\end{align} 
where the second inequality comes from the fact that $x^2-x\geq-\frac{1}{4}$ for all $x\in\RR$ and $C_2=(\sigma^2+1/\sigma^2)(2/\sigma^4+6/d_1+6/d_2+6\sigma^4)$ in the last inequality. We take $\eta\leq 1/C_2,$ then we have
\begin{align}
\EE\left[(\norm{x_{t+1}}_2^2+\norm{y_{t+1}}_2^2-1/\sigma^2)\mathds{1}_{\cH_t\cap\cA_t}\big|\cF_t\right]&\leq(1-2\eta\sigma^2)\left(\norm{x_{t}}_2^2+\norm{y_{t}}_2^2-1/\sigma^2\right)\mathds{1}_{\cH_t\cap\cA_t}\nonumber\\
&\leq(1-2\eta\sigma^2)\left(\norm{x_{t}}_2^2+\norm{y_{t}}_2^2-1/\sigma^2\right)\mathds{1}_{\cH_{t-1}\cap\cA_{t-1}}\label{martingale1}.
\end{align}
If we denote $G_t=(1-2\eta\sigma^2)^{-t}(\norm{x_{t+1}}_2^2+\norm{y_{t+1}}_2^2-1/\sigma^2),$  $G_t\mathds{1}_{\cH_{t-1}\cap\cA_{t-1}}$ is then a super-martingale according to \eqref{martingale1}. We will apply Azuma's Inequality to prove the bound and before that we have to bound the difference between $G_{t+1}\mathds{1}_{\cH_{t}\cap\cA_{t}}$ and $\EE[G_{t+1}\mathds{1}_{\cH_{t}\cap\cA_{t}}\big|\cF_t].$
\begin{align*}
d_{t+1}&=\big| G_{t+1}\mathds{1}_{\cH_{t}\cap\cA_{t}}- \EE[G_{t+1}\mathds{1}_{\cH_{t}\cap\cA_{t}}\big|\cF_t]\big|\\
&=(1-2\eta\sigma^2)^{-t-1}\Bigg|2\eta x_t^\top[(x_{t}+\xi_{1,t})(y_{t}+\xi_{2,t})^\top-M](y_{t}+\xi_{2,t}))\\&\hspace{+1.2in}-2\eta  \EE_{\xi_{1,t},\xi_{2,t}}\left[x_t^\top[(x_{t}+\xi_{1,t})(y_{t}+\xi_{2,t})^\top-M](y_{t}+\xi_{2,t}))\right]\\
&\hspace{+1.2in}+2\eta y_t^\top[(x_{t}+\xi_{1,t})(y_{t}+\xi_{2,t})^\top-M]^\top(x_{t}+\xi_{1,t}))\\&\hspace{+1.2in}-2\eta  \EE_{\xi_{1,t},\xi_{2,t}}\left[y_t^\top[(x_{t}+\xi_{1,t})(y_{t}+\xi_{2,t})^\top-M]^\top(x_{t}+\xi_{1,t}))\right]\\
&\hspace{+1.2in}+\eta^2\big|\big|\norm{y_t+\xi_{2,t}}_2^2(x_t+\xi_{1,t})-M(y_t+\xi_{2,t})\big|\big|_2^2\\
&\hspace{+1.2in}-\eta^2 \EE_{\xi_{1,t},\xi_{2,t}}\left[\big|\big|\norm{y_t+\xi_{2,t}}_2^2(x_t+\xi_{1,t})-M(y_t+\xi_{2,t})\big|\big|_2^2\right]\\
&\hspace{+1.2in}+\eta^2\big|\big|\norm{x_t+\xi_{1,t}}_2^2(y_t+\xi_{2,t})-M^\top(x_t+\xi_{1,t})\big|\big|_2^2\\
&\hspace{+1.2in}-\eta^2 \EE_{\xi_{1,t},\xi_{2,t}}\left[\big|\big|\norm{x_t+\xi_{1,t}}_2^2(y_t+\xi_{2,t})-M^\top(x_t+\xi_{1,t})\big|\big|_2^2\right]\Bigg|\mathds{1}_{\cH_{t}\cap\cA_{t}}
\\&\leq {C_1}' (1-2\eta\sigma^2)^{-t-1}\eta \sigma^{-2}\left(\left(\log\frac{d_1+d_2}{\eta}\right)^{\frac{1}{2}}+\left(\log\frac{1}{\delta}\right)^{\frac{1}{2}}\right),
\end{align*}
where ${C_1}'$ is some positive constant. Denote $r_t=\sqrt{\sum_{i=1}^t d_i^2}.$
By Azuma's inequality, we have
$$\PP\left(G_t\mathds{1}_{\cH_{t-1}\cap\cA_{t-1}}-G_0\geq  O(1)r_t\left(\log\frac{1}{\eta^2\delta}\right)^{\frac{1}{2}}\right)\leq O(\eta^2\delta).$$
Then when  with probability at least $1-O(\eta^2\delta),$ we have 
\begin{align*}
(\norm{x_{t+1}}_2^2+\norm{y_{t+1}}_2^2)\mathds{1}_{\cH_{t}\cap\cA_{t}}&\leq  1/\sigma^2+(1-\eta \sigma^2)^t(\norm{x_0}^2_2+\norm{y_0}^2_2-1/\sigma^2)\\
&\hspace{+0.3in}+O(1)(1-\eta\sigma^2)^t r_t(\log\frac{1}{\eta^2\delta})^{\frac{1}{2}}\\
&\leq 1/\sigma^2+ 0+O\left(\sqrt{\eta} \sigma^{-2}\left(\left(\log\frac{d_1+d_2}{\eta}\right)^{\frac{1}{2}}+\left(\log\frac{1}{\delta}\right)^{\frac{1}{2}}\right)\right)\left(\log\frac{1}{\eta^2\delta}\right)^{\frac{1}{2}}\\
&\leq 2/\sigma^2,
\end{align*}
when $\eta=O\left(\left(\log\frac{d_1+d_2}{\delta}\log\frac{1}{\delta}\right)^{-1}\right).$ In order to satisfy $\eta\leq1/C_2$ at the same time, we take $\eta\leq\eta_1=O\left({\sigma^6}\left(\log\frac{d_1+d_2}{\delta}\log\frac{1}{\delta}\right)^{-1}\right)$ to make sure that all inequalities above hold.

The above inequality shows that if $\cH_{t-1}\cap\cA_{t-1}$ holds, then $\cH_{t}\cap\cA_{t}$ holds with probability  at least $1-O(\eta^2\delta).$ Hence with probability at least $1-\delta,$  we have $(\norm{x_{t}}_2^2+\norm{y_{t}}_2^2)\mathds{1}_{\cA_{t-1}}\leq 2/\sigma^2$ for all $t\leq T_1=O(1/\eta^2).$ Recall that $$\PP(\cA_{1/\eta^2})\geq 1-\delta.$$ Thus,  we have  with probability at least $1-2\delta,$ $\norm{x_{t}}_2^2+\norm{y_{t}}_2^2\leq 2/\sigma^2$ for all $t\leq T_1=O(1/\eta^2).$ By properly rescaling  $\delta$, we prove Lemma \ref{lem_bounded}.
\end{proof}
\subsection{Proof of Lemma \ref{lem_b_converge}}
\begin{proof}
We only prove the convergence of $\norm{\beta_{1,t}}_2^2$ here. The proof of the convergence of $\norm{\beta_{2,t}}_2^2$ follows similar lines. For notational simplicity, we denote $\xi_{1,t}^{(-1)}=(\xi_{1,t}^{(2)},..., \xi_{1,t}^{(d_1)})^\top,\ \forall t\geq 0.$ We first bound the conditional expectation of $\norm{\beta_{1,t+1}}_2^2$ given $\cF_t$:
\begin{align}
\EE\left[\norm{\beta_{1,t+1}}_2^2\big|\cF_t\right]&=\norm{\beta_{1,t}}_2^2+\eta^2 \EE_{\xi_{1,t},\xi_{2,t}}\left[\norm{y_t+\xi_{2,t}}_2^4\norm{\beta_{1,t}+\xi_{1,t}^{(-1)}}_2^2\right]-2\eta\left(\norm{y_t}_2^2+\sigma^2\right)\norm{\beta_{1,t}}_2^2\nonumber\\
&\leq (1-2\eta\sigma^2)\norm{\beta_{1,t}}_2^2+\eta^2 C_2, \label{b_martingale}
\end{align}
where $C_2=(\sigma^2+1/\sigma^2)(2/\sigma^4+6/d_1+6/d_2+6\sigma^4).$ The bound on the second moment comes from the boundedness of $\norm{y_t}_2^2,$ which has been shown in Lemma \ref{lem_bounded}.
Define $G_t=(1-2\eta\sigma^2)^{-t}\left(\norm{\beta_{1,t}}_2^2-\frac{\eta C_2}{2\sigma^2}\right),$ and $\cE_t=\left\{\forall \tau\leq t, \norm{\beta_{1,\tau}}_2^2\geq \frac{\eta C_2}{\sigma^2}\right\}.$ By \eqref{b_martingale}, we have
$$\EE\left[G_{t+1}\mathds{1}_{\cE_t}\big|\cF_t\right]\leq  G_t\mathds{1}_{\cE_t}\leq G_t\mathds{1}_{\cE_{t-1}} .$$
Hence, by Markov inequality we have
\begin{align*}
\PP(\cE_t)= \PP\left( \norm{\beta_{1,t}}_2^2\mathds{1}_{\cE_{t-1}}\geq \frac{\eta C_2}{\sigma^2}\right)&\leq\frac{\EE\left[ \norm{\beta_{1,t}}_2^2\mathds{1}_{\cE_{t-1}}\right]}{\frac{\eta C_2}{\sigma^2}}\\
&\leq\frac{(1-2\eta\sigma^2)^{t}(\norm{\beta_{1,0}}_2^2-\frac{\eta C_2}{2\sigma^2})+\frac{\eta C_2}{2\sigma^2}}{\frac{\eta C_2}{\sigma^2}}\\
&\leq (1-2\eta\sigma^2)^{t}\frac{2}{C_2\eta}+\frac{1}{2}
\leq \frac{3}{4},
\end{align*}
when $t\geq\frac{1}{2\eta\sigma^2}\log\frac{8}{C_2\eta}.$ We take $t=\frac{1}{\eta\sigma^2}\log\frac{8}{C_2\eta}$ to make sure the inequality above holds.
Thus with probability at least $\frac{1}{4},$  there exists a $ \tau\leq\frac{1}{\eta\sigma^2}\log\frac{8}{C_2\eta}$,  such that $ \norm{\beta_{1,\tau}}_2^2\leq \frac{\eta C_2}{\sigma^2}.$ Thus, with probability at least $1-\delta,$ we can find a $\tau$ such that $ \norm{\beta_{1,\tau}}_2^2\leq \frac{\eta C_2}{\sigma^2},$ where $$\tau\leq\tau_1=\frac{1}{\log (4/3)\eta\sigma^2}\log\frac{8}{C_2\eta}\log\frac{1}{\delta}=\frac{1}{\log (4/3)\eta\sigma^2}\left(\log\frac{8}{C_2}+\log\frac{1}{\eta}\right)\log\frac{1}{\delta}=O\left(\frac{1}{\eta\sigma^2}\log\frac{1}{\eta}\log\frac{1}{\delta}\right).$$

We next show that with probability at least $1-\delta$, for $\forall t\geq \tau_1,$ we have $ \norm{\beta_{1,t}}_2^2\leq 2\frac{\eta C_2}{\sigma^2}.$ This can be done following the similar lines to the proof of Lemma \ref{lem_bounded}. We first restart the counter of the time  and assume   $ \norm{\beta_{1,0}}_2^2\leq \frac{\eta C_2}{\sigma^2}.$ Denote $\cH_t=\left\{\forall \tau\leq t,\norm{\beta_{1,\tau}}_2^2\leq \frac{2\eta C_2}{\sigma^2}\right\}.$
Then by \eqref{b_martingale}, we have
$$\EE\left[G_{t+1}\mathds{1}_{\cH_t\cap\cA_t}\big|\cF_t\right]\leq  G_t\mathds{1}_{\cH_t\cap\cA_t}\leq G_t\mathds{1}_{\cH_{t-1}\cap\cA_{t-1}} .$$
The difference of $ G_{t+1}\mathds{1}_{\cH_{t}\cap\cA_{t}}$ and $\EE[G_{t+1}\mathds{1}_{\cH_{t}\cap\cA_{t}}\big|\cF_t]$ can be easily bounded as follows
\begin{align*}
D_{t+1}&=\big| G_{t+1}\mathds{1}_{\cH_{t}\cap\cA_{t}}- \EE[G_{t+1}\mathds{1}_{\cH_{t}\cap\cA_{t}}\big|\cF_t]\big|={C_2}'\eta^2\sigma^{-3}\left(\left(\log\frac{d_1+d_2}{\eta}\right)^{\frac{1}{2}}+\left(\log\frac{1}{\delta}\right)^{\frac{1}{2}}\right),
\end{align*}
where ${C_2}'$ is some constant. Then by applying Azuma's inequality and following the similar lines to the proof of Lemma \ref{lem_bounded}, we show that when 
$$\eta\leq {C_3}'{\sigma^8}\left(\log\frac{d_1+d_2}{\delta}\log\frac{1}{\delta}\right)^{-1},$$ 
with probability at least $1-\delta,$ we have 
\begin{align*}
\norm{\beta_{1,t}}_2^2\leq  2\frac{\eta C_2}{\sigma^2}.\end{align*}
for any $\tau_1\leq t\leq T_1= O(1/\eta^2).$
\end{proof}
\subsection{Proof of Lemma \ref{lem_escape}}
We prove this lemma in three steps. 

\noindent $\bullet$ {\bf Step 1:} The following lemma shows that after polynomial time, with high probability, the algorithm can move out of the $O(\eta)$ neighborhood of the saddle point. 
\begin{lemma}[Escaping from the Unique Saddle Point]
Suppose $\norm{\beta_{1,t}}_2^2\leq  2\frac{\eta C_2}{\sigma^2}$ and $\norm{\beta_{2,t}}_2^2\leq  2\frac{\eta C_2}{\sigma^2}$ hold for all $t>0.$ For  $\forall\delta\in(0,1),$ we take $$\eta=O\left(\sigma^{12}\left(\log\frac{d_1+d_2}{\delta}\log\frac{1}{\delta}\right)^{-1}\right).$$
Then with probability at least $1-\delta,$ there exists a $\tau\leq{\tau}_2,$ such that $$x_{\tau}^\top M y_{\tau}\geq  9\frac{\eta C_2}{\sigma^4},$$ where ${\tau}_2=\frac{5}{\eta \sigma^2}\log \frac{4 \sigma^3}{\eta C_2}\log\frac{1}{\delta}.$
\end{lemma}
\begin{proof}
Let's define the event  $\cH_t=\{x_{\tau}^\top M y_{\tau}\leq9\frac{\eta C_2}{\sigma^4}, ~\forall \tau\leq t \}.$ Following the proof of Lemma \ref{lem_bounded}, we can refine the bound on the conditional expectation of $\norm{x_{t+1}}^2_2+\norm{y_{t+1}}^2_2$ given $\cH_t.$ Specifically, we have 
\begin{align*}
\EE\left[(\norm{x_{t+1}}_2^2+\norm{y_{t+1}}_2^2)\mathds{1}_{\cH_t}\big|\cF_t\right]=&\left\{(1-2\eta\sigma^2)\left(\norm{x_{t}}_2^2+\norm{y_{t}}_2^2\right)-4\eta\left(\norm{x_{t}}_2^2\norm{y_{t}}_2^2-x_t^\top M y_t\right)\right\}\mathds{1}_{\cH_t}\nonumber\\
&+\eta^2 \EE_{\xi_{1,t},\xi_{2,t}}\left[\norm{\nabla_{x}\cF(x_{t}+\xi_{1,t},y_{t}+\xi_{2,t})}_2^2+\norm{\nabla_{y}\cF(x_{t}+\xi_{1,t},y_{t}+\xi_{2,t})}_2^2\right]\mathds{1}_{\cH_t}\nonumber\\
\leq&\left\{(1-2\eta\sigma^2)\left(\norm{x_{t}}_2^2+\norm{y_{t}}_2^2\right)+34C_2\eta^2 \sigma^{-4})\right\}\mathds{1}_{\cH_t}.
\end{align*} 
Denote $G_t=(1-2\eta\sigma^2)^{-t}\left(\norm{x_{t}}_2^2+\norm{y_{t}}_2^2-17C_2\eta \sigma^{-6}\right).$ Then we have
$$\EE[G_{t+1}\mathds{1}_{\cH_t}]\leq G_{t}\mathds{1}_{\cH_t}\leq G_{t}\mathds{1}_{\cH_{t-1}}.$$
Thus, by Markov inequality, we have
\begin{align*}
\PP( (\norm{x_{t}}_2^2+\norm{y_{t}}_2^2)\mathds{1}_{\cH_{t-1}}\geq 34C_2\eta \sigma^{-6} )&\leq \frac{(1-2\eta\sigma^2)^{t}(\norm{x_{0}}_2^2+\norm{y_{0}}_2^2-17C_2\eta \sigma^{-6})+17C_2\eta \sigma^{-6}}{34C_2\eta \sigma^{-6}}\\
&\leq \frac{3}{4},
\end{align*}
when $t\geq \frac{1}{2\eta \sigma^2}\log \frac{4 \sigma^3}{\eta C_2}.$ Thus with probability at least $1-\delta,$ there exists a $\tau\leq{\tau}_2=\frac{5}{2\eta \sigma^2}\log \frac{4 \sigma^3}{\eta C_2}\log\frac{1}{\delta}=O(\frac{1}{\eta \sigma^2}\log \frac{1}{\eta }\log\frac{1}{\delta}),$ such that $$(\norm{x_{\tau}}_2^2+\norm{y_{{\tau}}}_2^2)\mathds{1}_{\cH_{\tau-1}}\leq 34C_2\eta \sigma^{-6}.$$ Following the exactly same proof of Lemma \ref{lem_bounded}, we can show that with probability at least $1-\delta,$ for all $\tau_2\leq t\leq T_1=O(\frac{1}{\eta^2}),$ we have \begin{align*}
(\norm{x_{t}}_2^2+\norm{y_{t}}_2^2)\mathds{1}_{\cH_{t-1}}&\leq(17+34)C_2\eta \sigma^{-6}+O\left(\sqrt{\eta} \sigma^{-2}\left(\left(\log\frac{d_1+d_2}{\eta}\right)^{\frac{1}{2}}+\left(\log\frac{1}{\delta}\right)^{\frac{1}{2}}\right)\right)\left(\log\frac{1}{\eta^2\delta}\right)^{\frac{1}{2}}
\\&= {C_4}'\eta\sigma^{-12}+{C_5}'\sqrt{\eta} \sigma^{-2}\left(\left(\log\frac{d_1+d_2}{\eta}\right)^{\frac{1}{2}}+\left(\log\frac{1}{\delta}\right)^{\frac{1}{2}}\right)\left(\log\frac{1}{\eta^2\delta}\right)^{\frac{1}{2}}\\
&\leq {C_6}' \left(\log\frac{d_1+d_2}{\delta}\log\frac{1}{\delta}\right)^{-1}+{C_7}' \sigma^{4},
\end{align*}
where ${C_4}'$, ${C_5}'$, ${C_6}'$ and ${C_7}'$ are some positive constants, and when$$\eta= O\left({\sigma^{12}}\left(\log\frac{d_1+d_2}{\delta}\log\frac{1}{\delta}\right)^{-1}\right).$$
We next show that for large enough $t,$ with high probability, $\cH_{t-1}$ does not hold. We prove by contradiction: if  $\cH_{t-1}$ holds for all $t$, the solution trajectory stays in a small neighborhood around $0.$ If so, we can show that with constant probability, $|\alpha_{1,t}+\alpha_{2,t}|$ will explode to infinity, which is in contradiction with the boundedness. Here follows the detailed proof.

Assuming $\cH_{t-1}$ holds for all $t\leq\tilde{\tau}_2=O(\frac{1}{\eta}\log\frac{1}{\eta\sigma}\log\frac{1}{\delta})$, by the analysis above we have $\norm{x_{t}}_2^2+\norm{y_{t}}_2^2\leq {C_6}' \left(\log\frac{d_1+d_2}{\delta}\log\frac{1}{\delta}\right)^{-1}+{C_7}' \sigma^{4}$ holds for $\forall t\leq\tilde{\tau}_2$. 

Note that with at least some constant probability, \begin{align*}&\Big|\alpha_{1,t+1}+\alpha_{2,t+1}\Big|-\Big|\alpha_{1,t}+\alpha_{2,t}\Big|\\=&\Big|\alpha_{1,t}+\alpha_{2,t}-\eta\left(<\nabla_{x}\cF(x_{t}+\xi_{1,t},y_{t}+\xi_{2,t}),u_*>+<\nabla_{y}\cF(x_{t}+\xi_{1,t},y_{t}+\xi_{2,t}),v_*>\right)\Big|-\Big|\alpha_{1,t}+\alpha_{2,t}\Big| \\ \geq &{C_8}'\eta(\sigma_1+\sigma_2),\end{align*} where ${C_8}'$ is some positive constant. This means we can find a $\tau=O(\log\frac{1}{\delta})\leq\tilde{\tau}_2$, such that $\Big|\alpha_{1,\tau}+\alpha_{2,\tau}\Big|>{C_8}'\eta(\sigma_1+\sigma_2),$ with probability at least $1-\delta.$ We will use this point as our initialization for the following proof.
We next give a lower bound on the conditional expectation of $|\alpha_{1,t+1}+\alpha_{2,t+1}|.$
\begin{align*}
&\EE\left[|\alpha_{1,t+1}+\alpha_{2,t+1}|\big|\cF_t\right]\\
\geq&\Big|\EE\left[(\alpha_{1,t+1}+\alpha_{2,t+1})\big|\cF_t\right]\Big|\\
=&\Big|\left(1+\eta(1-\sigma^2-\alpha_{1,t}\alpha_{2,t})\right)(\alpha_{1,t}+\alpha_{2,t})-\eta\left(\norm{\beta_{1,t}}_2^2\alpha_{2,t}+\norm{\beta_{2,t}}_2^2\alpha_{1,t}\right)\Big|\\
\geq& \left(1+\eta(1-\sigma^2-\alpha_{1,t}\alpha_{2,t})\right)\Big|\alpha_{1,t}+\alpha_{2,t}\Big|- 2{C_9}' \frac{\eta^{2.2} C_2}{\sigma^3}\left(\left(\log\frac{d_1+d_2}{\eta}\right)^{\frac{1}{2}}+\left(\log\frac{1}{\delta}\right)^{\frac{1}{2}}\right)\\
\geq&\left(1+\frac{1}{2}\eta\right)\Big|\alpha_{1,t}+\alpha_{2,t}\Big|- 2{C_9}'\frac{\eta^{2.2} C_2}{\sigma^3}\left(\left(\log\frac{d_1+d_2}{\eta}\right)^{\frac{1}{2}}+\left(\log\frac{1}{\delta}\right)^{\frac{1}{2}}\right),
\end{align*}
where ${C_9}'$ is some positive constant.
This is equivalent to the following inequality:
\begin{align*}
&\EE\left[|\alpha_{1,t+1}+\alpha_{2,t+1}|-4{C_9}'\frac{\eta^{1.2} C_2}{\sigma^3}\left(\left(\log\frac{d_1+d_2}{\eta}\right)^{\frac{1}{2}}+\left(\log\frac{1}{\delta}\right)^{\frac{1}{2}}\right)\Bigg|\cF_t\right]\\\geq &\left(1+\frac{1}{2}\eta\right)\left(\Big|\alpha_{1,t}+\alpha_{2,t}\Big|- 4{C_9}'\frac{\eta^{1.2} C_2}{\sigma^3}\left(\left(\log\frac{d_1+d_2}{\eta}\right)^{\frac{1}{2}}+\left(\log\frac{1}{\delta}\right)^{\frac{1}{2}}\right)\right)\\
\geq &\left(1+\frac{1}{2}\eta\right)^{t+1}\left(\Big|\alpha_{1,0}+\alpha_{2,0}\Big|- 4{C_9}'\frac{\eta^{1.2} C_2}{\sigma^3}\left(\left(\log\frac{d_1+d_2}{\eta}\right)^{\frac{1}{2}}+\left(\log\frac{1}{\delta}\right)^{\frac{1}{2}}\right)\right)\\
\geq &\left(1+\frac{1}{2}\eta\right)^{t+1}\left({C_8}'\eta(d_1+d_2)- 4{C_9}'\frac{\eta^{1.2} C_2}{\sigma^3}\left(\left(\log\frac{\sigma_1+\sigma_2}{\eta}\right)^{\frac{1}{2}}+\left(\log\frac{1}{\delta}\right)^{\frac{1}{2}}\right)\right)\\
\geq &\ {C_{10}}' \eta(\sigma_1+\sigma_2)\exp{t\eta/2}\\
\geq & \ {C_{10}}' \eta(\sigma/\sqrt{d_1}+\sigma/\sqrt{d_2})\frac{1}{\eta\sigma} \\
= &\  {C_{10}}'(1/\sqrt{d_1}+1/\sqrt{d_2})  ,
\end{align*}
where ${C_{10}}' $ is some positive constant and last inequality holds when $t\geq \frac{2}{\eta}\log\frac{1}{\eta\sigma}.$ Note that here we still take $\eta=O\left(\sigma^{12}\left(\log\frac{d_1+d_2}{\delta}\log\frac{1}{\delta}\right)^{-1}\right)$, which makes sure that $\frac{\eta^{1.2} C_2}{\sigma^3}\left(\left(\log\frac{d_1+d_2}{\eta}\right)^{\frac{1}{2}}+\left(\log\frac{1}{\delta}\right)^{\frac{1}{2}}\right)$ is small enough to make the fourth inequality hold.

For fixed $d_1$ and $d_2$, if we let $\delta$ and $\sigma$ go to zero (which guarantees that $\eta$ goes to zero), we will have the conditional expectation of $|\alpha_{1,t+1}+\alpha_{2,t+1}|$ stays in a neighborhood around a positive constant. However, by our assumption, $\norm{x_{t+1}}_2^2+\norm{y_{t+1}}_2^2$ stays in a very small neighborhood around zero, which makes $\alpha_{1,t+1}^2+\alpha_{2,t+1}^2\leq\norm{x_{t+1}}_2^2+\norm{y_{t+1}}_2^2$ also very small. This implies that $|\alpha_{1,t+1}+\alpha_{2,t+1}|$ can be arbitrarily small as long as we make $\delta$ and $\sigma$ small enough, which is the desired contradiction.


Thus, we know that, with probability at least $1-\delta,$ there exists a t satisfying ${\tau}_2\leq t\leq {\tau}_2+\tilde{\tau}_2\leq 2{\tau}_2,$ such that
$\cH_{t}$ does not hold. Re-scale ${\tau}_2$ and we prove the result.
\end{proof}
\noindent $\bullet$ {\bf Step 2:} The following lemma shows that after step 1, with high probability, the algorithm will continue move away from the saddle point and escape from the saddle point at some time.
\begin{lemma} \label{lem_conv_region}
Suppose $x_0^\top M y_0\geq  9\frac{\eta C_2}{\sigma^4}.$ We take $\eta$ as in Lemma \ref{lem_escape}. Then there almost surely exists a $\tau\leq \tau_2'=\frac{2}{\eta}\log\frac{2\sigma^3}{\eta C_2},$ such that $$x_{\tau}^\top M y_{\tau}\geq\frac{1}{2}+\sigma^2.$$
\end{lemma}
\begin{proof}
Suppose  $x_t^\top M y_t\leq\frac{1}{2}+\sigma^2$ holds for all $t\leq \tau_2'.$
We next give a lower bound on the conditional expectation of $|\alpha_{1,t+1}+\alpha_{2,t+1}|.$
\begin{align*}
\EE\left[|\alpha_{1,t+1}+\alpha_{2,t+1}|\big|\cF_t\right]&\geq\Big|\EE\left[(\alpha_{1,t+1}+\alpha_{2,t+1})\big|\cF_t\right]\Big|\\
&=\Big|\left(1+\eta(1-\sigma^2-\alpha_{1,t}\alpha_{2,t})\right)(\alpha_{1,t}+\alpha_{2,t})-\eta\left(\norm{\beta_{1,t}}_2^2\alpha_{2,t}+\norm{\beta_{2,t}}_2^2\alpha_{1,t}\right)\Big|\\
&\geq\left(1+\frac{1}{2}\eta\right)\Big|\alpha_{1,t}+\alpha_{2,t}\Big|-  4\frac{\eta^2 C_2}{\sigma^4}.\end{align*}
This implies that
\begin{align*}
\EE\left[|\alpha_{1,t+1}+\alpha_{2,t+1}|-8\frac{\eta C_2}{\sigma^4}\big|\cF_t\right]&\geq\left(1+\frac{1}{2}\eta\right)\left(\Big|\alpha_{1,t}+\alpha_{2,t}\Big|-8\frac{\eta C_2}{\sigma^4}\right).\end{align*}
The above inequality further implies a lower bound on the conditional expectation of $|\alpha_{1,t+1}+\alpha_{2,t+1}|-8\frac{\eta C_2}{\sigma^4}.$
\begin{align*}
\EE\left[|\alpha_{1,t+1}+\alpha_{2,t+1}|-8\frac{\eta C_2}{\sigma^4}\right]&\geq\left(1+\frac{1}{2}\eta\right)^{t}\left(\Big|\alpha_{1,0}+\alpha_{2,0}\Big|-8\frac{\eta C_2}{\sigma^4}\right)\\
&\geq\left(1+\frac{1}{2}\eta\right)^{t} \frac{\eta C_2}{\sigma^4}\geq \frac{2}{\sigma},\end{align*}
when $t= \tau_2'=\frac{2}{\eta}\log\frac{2\sigma^3}{\eta C_2}.$ On the other hand, $$\EE\left[|\alpha_{1,t+1}+\alpha_{2,t+1}|-8\frac{\eta C_2}{\sigma^4}\right]\leq  \frac{\sqrt{2}}{\sigma},$$ which leads to a contradiction unless
$$\PP(x_t^\top M y_t\leq\frac{1}{2}+\sigma^2,~\forall t\leq \tau_2')=0.$$ We prove the result.
\end{proof}
\noindent $\bullet$ {\bf Step 3:} We then show that the algorithm will never iterate back towards the saddle point after escaping from it. Then Lemma \ref{lem_escape} is proved.
\begin{lemma}\label{prop1}
Suppose there exists a time step $\tau$ such that for some positive constant $c<\frac{1}{2}$
\begin{align*}
    x_{\tau}^\top My_{\tau}>2c,
\end{align*}
then $\forall \delta \in (0,1)$, with probability at least $1-\delta$, $\forall \tau\leq t\leq T_1\triangleq O(\frac{1}{\eta^2})$,
\begin{align*}
    (x_t,y_t) \in \left\{(x,y)| x^\top My>c\right\}.
\end{align*}
\end{lemma}

By taking $c=\frac{1}{4}$, we can prove Lemma \ref{lem_escape} from Lemma \ref{prop1}. 
\begin{proof}

We consider two cases: \\
(i) if $\alpha_{1,t}\alpha_{2,t}\geq2(1-c)$, then $\alpha_{1,t+1}\alpha_{2,t+1}>2c$ w.p.1.\\
(ii) if $\alpha_{1,t}\alpha_{2,t}<2(1-c)$, then, $\forall \delta \in (0,1)$ and $\forall t \leq T_1=O(\frac{1}{\eta^2})$, w.p. at least $1-\delta$,
\begin{align*}(\alpha_{1,t},\alpha_{2,t})\in\{(\alpha_1,\alpha_2)|\left(\alpha_1\alpha_2-1\right)^2<(1-c)^2\}.
\end{align*}

Note that $(\alpha_{1,t}\alpha_{2,t}-1)^2<(1-c)^2$ implies that $\alpha_{1,t}\alpha_{2,t}>c$, thus we can prove Lemma \ref{prop1}. Further note that the injected noise is upper bounded with high probability, we assume in the following $T_1=O(\frac{1}{\eta^2})$ steps, 
\begin{align*}
\norm{\xi_{k,t}}_\infty\leq\Bar{\sigma}\triangleq\sigma\left(\sqrt{2\log((d_1+d_2)\eta^{-2})} +\sqrt{\log(1/\delta)}\right).
\end{align*}

{\bf Case (i):} By plugging \eqref{*4} and \eqref{*5} to \eqref{*1}, we get 
\begin{align*}
\alpha_{1,t+1}\alpha_{2,t+1}=\alpha_{1,t}\alpha_{2,t}-\eta \left(\alpha_{2,t} (\alpha_{1,t}(\alpha_{2,t}^2+\norm{\beta_{2,t}}_2^2)-\alpha_{2,t}+g_x)+\alpha_{1,t}(\alpha_{2,t}(\alpha_{1,t}^2+\norm{\beta_{1,t}}_2^2)-\alpha_{1,t}+g_y)\right)
\\+\eta^2(\alpha_{1,t}(\alpha_{2,t}^2+\norm{\beta_{2,t}}_2^2)-\alpha_{2,t}+g_x)(\alpha_{2,t}(\alpha_{1,t}^2+\norm{\beta_{1,t}}_2^2)-\alpha_{1,t}+g_y)\\=2c+(\alpha_{1,t}\alpha_{2,t}-2c)-\eta [\alpha_{2,t} (\alpha_{1,t}(\alpha_{2,t}^2+\norm{\beta_{2,t}}_2^2)-\alpha_{2,t}+g_x)+\alpha_{1,t}(\alpha_{2,t}(\alpha_{1,t}^2+\norm{\beta_{1,t}}_2^2)-\alpha_{1,t}+g_y)
\\+\eta(\alpha_{1,t}(\alpha_{2,t}^2+\norm{\beta_{2,t}}_2^2)-\alpha_{2,t}+g_x)(\alpha_{2,t}(\alpha_{1,t}^2+\norm{\beta_{1,t}}_2^2)-\alpha_{1,t}+g_y)],
\end{align*}

First note that by Lemma \ref{lem_conv_region} we can always find such $0<c<\frac{1}{2}$ in the condition. 
As $\alpha_{1,t}\alpha_{2,t}-2c>2(1-c)-2c=2(1-2c)>0$, we can prove (i) by choosing $\eta$ small enough.

By Lemma \ref{lem_bounded}, $\alpha_{1,t}^2+\alpha_{2,t}^2$ is bounded by $\frac{2}{\sigma^2}$ and we can know $g_x$ and $g_y$ are at most of order $O\left(\frac{1}{\sigma}\left(\sqrt{\log\frac{d_1+d_2}{\eta^2}}+\sqrt{\log\frac{1}{\delta}}\right)\right)$ from the following upper bound:
\begin{align*}
|g_x|\leq\Bar{\sigma}(\alpha_{2,t}^2+2\alpha_{2,t}\alpha_{1,t}+1+\frac{2\eta C_2}{\sigma^2})+\Bar{\sigma}^2(2\alpha_{2,t}+\alpha_{1,t})+\Bar{\sigma}^3,\\|g_y|\leq\Bar{\sigma}(\alpha_{1,t}^2+2\alpha_{1,t}\alpha_{2,t}+1+\frac{2\eta C_2}{\sigma^2})+\Bar{\sigma}^2(2\alpha_{1,t}+\alpha_{2,t})+\Bar{\sigma}^3.
\end{align*}
It is easy to show that for properly selected $$\eta\leq\tilde{\eta}_1=O\left(\left(\frac{1}{\sigma^4}+\frac{1}{\sigma^2}\left(\sqrt{\log(d_1+d_2)}+\sqrt{\log\frac{1}{\delta}}\right)\right)^{-1}\right),$$ the last two terms are greater than zero. Thus, w.p.1, $\alpha_{1,t+1}\alpha_{2,t+1}>2c$.\\
{\bf Case (ii):} Without loss of generality, we place the time origin at t, and thus we have $2c<\alpha_1^0\alpha_2^0<2-2c$.
\begin{align*}
 \EE_\xi[A_t]
 &=\EE_\xi[\alpha_{2,t}<\nabla_{x}\cF(x_{t}+\xi_{1,t},y_{t}+\xi_{2,t}),u_*>+\alpha_{1,t}<\nabla_{y}\cF(x_{t}+\xi_{1,t},y_{t}+\xi_{2,t}),v_*>].
\end{align*}
By plugging \eqref{*4} and \eqref{*5} in the above equation, we have
\begin{align*}
 \EE_\xi[A_t]
 &=\alpha_{2,t}(\alpha_{1,t}(\alpha_{2,t}^2+\norm{\beta_{2,t}}_2^2)-\alpha_{2,t}+\alpha_{1,t}\sigma^2)+\alpha_{1,t}(\alpha_{2,t}(\alpha_{1,t}^2+\norm{\beta_{1,t}}_2^2)-\alpha_{1,t}+\alpha_{2,t}\sigma^2)
 \\&=(\alpha_{2,t}\alpha_{1,t}-1)(\alpha_{2,t}^2+\alpha_{1,t}^2)+\alpha_{2,t}\alpha_{1,t}(2\sigma^2+\norm{\beta_{1,t}}_2^2+\norm{\beta_{2,t}}_2^2).
\end{align*}
Again, by plugging \eqref{*4} and \eqref{*5} to \eqref{*2} and taking expectation conditioning on $\cF_t$, when $\alpha_{1,t}\alpha_{2,t}>c$, we will get
\begin{align*}
\EE[(\alpha_{1,t+1}\alpha_{2,t+1}-1)^2|\cF_t]
\nonumber&=(\alpha_{1,t}\alpha_{2,t}-1)^2-2\eta \EE_\xi[A_t](\alpha_{1,t}\alpha_{2,t}-1)\\&+\eta^2 \EE_\xi[A_t^2]+\eta^4 \EE_\xi[B_t^2]-2\eta^3\EE_\xi[A_tB_t]+2\eta^2\EE_\xi[B_t](\alpha_{1,t}\alpha_{2,t}-1)\\
&=\left(1-2\eta (\alpha_{2,t}^2+\alpha_{1,t}^2)\right)(\alpha_{1,t}\alpha_{2,t}-1)^2\\&~~~~-2\eta \alpha_{2,t}\alpha_{1,t}(\alpha_{1,t}\alpha_{2,t}-1)(2\sigma^2+\norm{\beta_{1,t}}_2^2+\norm{\beta_{2,t}}_2^2)\\&~~~~+\eta^2 \EE_\xi[A_t^2]+\eta^4 \EE_\xi[B_t^2]-2\eta^3\EE_\xi[A_tB_t]+2\eta^2\EE_\xi[B_t](\alpha_{1,t}\alpha_{2,t}-1)\\
&\leq\left(1-4\eta c\right)(\alpha_{1,t}\alpha_{2,t}-1)^2+2\eta \frac{1}{4}(2\sigma^2+\frac{2\eta C_2}{\sigma^2}+\frac{2\eta C_2}{\sigma^2})+\tilde{C}_1\eta^2\\
&\leq\left(1-4\eta c\right)(\alpha_{1,t}\alpha_{2,t}-1)^2+\eta (\sigma^2+\frac{2\eta C_2}{\sigma^2})+\tilde{C}_1\eta^2,
\end{align*}
where $\tilde{C}_1=O\left(\frac{1}{\sigma^4}\left(\log\frac{d_1+d_2}{\delta}\right)\right),$ if $\alpha_{1,t}\alpha_{2,t}$ is of constant order. Note that here we choose $\eta\leq\tilde{\eta}_1$ as mentioned in (i). Denote $\gamma\triangleq\frac{\eta (\sigma^2+\frac{2\eta C_2}{\sigma^2})+\tilde{C}_1\eta^2}{4\eta c}$, the inequality above can be re-expressed as
\begin{align} \label{*6}
\EE[\{(\alpha_{1,t+1}\alpha_{2,t+1}-1)^2-\gamma\}|\cF_t]\leq\left(1-4\eta c\right)\{(\alpha_{1,t}\alpha_{2,t}-1)^2-\gamma\}.
\end{align}

We denote $G_t\triangleq(1-4\eta c)^{-t}\{(\alpha_{1,t}\alpha_{2,t}-1)^2-\gamma\}$ and $\cE_t\triangleq\{\forall \tau\leq t : (\alpha_{1,\tau}\alpha_{2,\tau}-1)^2<(1-c)^2\}.$ Since $\alpha_{1,\tau}\alpha_{2,\tau}>c$ can be inferred from $(\alpha_{1,\tau}\alpha_{2,\tau}-1)^2<(1-c)^2$, we can get
\begin{align*}
\EE[G_{t+1}\mathds{1}_{\cE_t}|\cF_t]\leq G_{t}\mathds{1}_{\cE_t}\leq G_{t}\mathds{1}_{\cE_{t-1}}.
\end{align*}

This means $\{G_{t}\mathds{1}_{\cE_{t-1}}\}$ is a supermartingale. To use Azuma's inequality, we need to bound the following difference
\begin{align*}
d_{t+1}&\triangleq|G_{t+1}\mathds{1}_{\cE_t}-\EE[G_{t+1}\mathds{1}_{\cE_t}|\cF_t]|\\
&=(1-4\eta c)^{-t-1}|(\alpha_{1,t+1}\alpha_{2,t+1}-1)^2-\EE[(\alpha_{1,t+1}\alpha_{2,t+1}-1)^2|\cF_t]|\mathds{1}_{\cE_t}\\
&\leq(1-4\eta c)^{-t-1}\tilde{C}_2,
\end{align*}
where $\tilde{C}_2=O\left(\eta\frac{1}{\sigma^2}\left(\sqrt{\log\frac{d_1+d_2}{\eta^2}}+\sqrt{\log\frac{1}{\delta}}\right)\right)$. Again, here we choose $\eta\leq\tilde{\eta}_1$. We further define $r_t\triangleq\sqrt{\sum_{i=1}^t d_i^2}$, and by Azuma's inequality we have
\begin{align*}
\mathbb{P}\left(G_{t} \mathds{1}_{\cE_{t-1}}-G_{0} \geq O(1) r_{t} \log ^{\frac{1}{2}}\left(\frac{1}{\eta^{2} \delta}\right)\right) \leq \exp \left(-\frac{O(1) r_{t}^{2} \log \left(\frac{1}{\eta^{2} \delta}\right)}{2 \sum_{i=0}^{t} d_{i}^{2}}\right)=O\left(\eta^{2} \delta\right).
\end{align*}
Thus, with probability at least $1-O(\eta^2 \delta)$,
\begin{align*}
((\alpha_{1,t}\alpha_{2,t}-1)^2-\gamma) \mathds{1}_{\cE_{t-1}} &< (1-4\eta c)^t\left((\alpha_{1,0}\alpha_{2,0}-1)^2+O(1) r_{t} \log ^{\frac{1}{2}}\left(\frac{1}{\eta^{2} \delta}\right)\right)\\
&<(\alpha_{1,0}\alpha_{2,0}-1)^2+O\left( (1-4\eta c)^t r_{t} \log ^{\frac{1}{2}}\left(\frac{1}{\eta^{2} \delta}\right)\right)\\
&<(1-c)^2+(1-2c)^2-(1-c)^2\\&~~~~~~~~+O\left( \frac{\sqrt{\eta}}{\sigma^2}\left(\sqrt{\log\frac{d_1+d_2}{\eta^2}}+\sqrt{\log\frac{1}{\delta}}\right)\log ^{\frac{1}{2}}\left(\frac{1}{\eta^{2} \delta}\right)\right),
\end{align*}

When $\cE_{t-1}$ holds, w.p. at least $1-O(\eta^2 \delta)$, we will have
\begin{align*}
(\alpha_{1,t}\alpha_{2,t}-1)^2&<(1-c)^2+(1-2c)^2-(1-c)^2\\&~~~~+O\left( \frac{\sqrt{\eta}}{\sigma^2}\left(\sqrt{\log\frac{d_1+d_2}{\eta^2}}+\sqrt{\log\frac{1}{\delta}}\right)\log ^{\frac{1}{2}}\left(\frac{1}{\eta^{2} \delta}\right)\right)
+\gamma<(1-c)^2,
\end{align*}
as $(1-2c)^2-(1-c)^2$ is some negative constant, by choosing $\eta\leq\tilde{\eta}_2=O\left(\sigma^4(\log\frac{1}{\delta})^{-1}\left(\log\frac{d_1+d_2}{\delta}\right)^{-1}\right)$ and $\sigma$ small enough, we can make sure the sum of last four terms is negative. Now we know that if $\cE_{t-1}$ holds, $\cE_{t}$ holds w.p. at least $1-O(\eta^2 \delta)$. It is easy to show that the Perturbed GD updates satisfy $(\alpha_{1,t}\alpha_{2,t}-1)^2<(1-c)^2$ in the following $O(\frac{1}{\eta^2})$ steps w.p. at least $1-\delta$. 
\end{proof}
\subsection{Proof of Lemma \ref{lem_loss}}
We partition this lemma into two parts: Lemma \ref{lma11} shows that after polynomial time, the algorithm enters $\{ (x,y)\big|(x^\top My-1)^2<4\gamma\},$ where $\gamma$ is a small constant depending on $\sigma$, and Lemma \ref{lma22} shows that the algorithm then stays in $\{ (x,y)\big|(x^\top My-1)^2<6\gamma\}.$  It is easy to prove Lemma \ref{lem_loss} from Lemmas \ref{lma11} and \ref{lma22}.

\begin{lemma} \label{lma11}
Suppose $\forall t\leq T_1\triangleq O(\frac{1}{\eta^2})$,
\begin{align*}
    (x_t,y_t) \in \left\{(x,y)| x^\top My>c\right\},
\end{align*}then $\forall \delta \in (0,1)$, with probability at least $1-\delta$, there exists a time step $\tau\leq \tau_3\triangleq O(\frac{1}{\eta}\log\frac{1}{\sigma}\log\frac{1}{\delta})$ such that
\begin{align*}
    (x_{\tau}^\top My_{\tau}-1)^2<4\gamma,
\end{align*}
where {$\gamma=O(\sigma^2)$}.
\end{lemma}

\begin{lemma} \label{lma22}
Suppose there exists a time step $\tau\leq \tau_3\triangleq O(\frac{1}{\eta}\log\frac{1}{\sigma}\log\frac{1}{\delta})$ such that
\begin{align*}
    (x_{\tau}^\top My_{\tau}-1)^2<4\gamma,
\end{align*} and $\forall t\leq T_1\triangleq O(\frac{1}{\eta^2})$,
\begin{align*}
    (x_t,y_t) \in \left\{(x,y)| x^\top My>c\right\},
\end{align*}then $\forall \delta \in (0,1)$, with probability at least $1-\delta$,$\forall t\leq T_1\triangleq O(\frac{1}{\eta^2})$,
\begin{align*}
    (x_t,y_t) \in \left\{(x,y)| (x^\top My-1)^2<6\gamma\right\}.
\end{align*}
\end{lemma}

Here follows proof of Lemmas \ref{lma11} and \ref{lma22}.

\begin{proof}
Define $\cH_t\triangleq\{\forall \tau\leq t : (\alpha_{1,\tau}\alpha_{2,\tau}-1)^2\geq4\gamma\}$, for $t>\frac{\log(\frac{(\alpha_{1,0}\alpha_{2,0}-1)^2}{\gamma})}{4\eta c}$, we have
\begin{align*}
    4\gamma\EE[\mathds{1}_{\cH_t}]\leq\EE[(\alpha_{1,t}\alpha_{2,t}-1)^2-\gamma]\leq(1-4\eta c)^t\left((\alpha_{1,0}\alpha_{2,0}-1)^2-\gamma\right)+\gamma<2\gamma,
\end{align*}
where the first inequality comes from the definition of $\cH_t$ and the second one comes from \eqref{*6}. Thus, if we choose $t=O\left(\frac{\log(\frac{1}{\gamma \sigma^2})}{\eta}\right)$ and recursively applying the inequality above $O(\log(\frac{1}{\delta}))$ times,
we will get, for $\tau_3=O\left(\frac{1}{\eta}\log(\frac{1}{\delta})\log(\frac{1}{\gamma \sigma^2})\right)=O\left(\frac{1}{\eta}\log\frac{1}{\sigma}\log\frac{1}{\delta}\right)$,
\begin{align*}
    \PP(\cH_{\tau_3})<(\frac{1}{2})^{\log(\frac{1}{\delta})}=\delta.
\end{align*}
Thus, w.p. at least $1-\delta$, there exists a $\tau\leq \tau_3$ s.t. $(\alpha_{1,\tau}\alpha_{2,\tau}-1)^2<4\gamma$. Here, we finish the proof of Lemma \ref{lma11}.

Without loss of generality, we place the time origin at $\tau$, i.e. $(\alpha_{1,0}\alpha_{2,0}-1)^2<4\gamma$. We next prove Lemma \ref{lma22}.
Denote $\cA\triangleq\{(\alpha_1,\alpha_2)|(\alpha_1\alpha_2-1)^2<6\gamma\}$ and $\cA_t\triangleq\{\forall \tau\leq t:(\alpha_{1,\tau},\alpha_{2,\tau})\in \cA\}$, again note that $\alpha_{1,t}\alpha_{2,t}>c$ can be inferred from $(\alpha_{1,t}\alpha_{2,t}-1)^2<6\gamma<(1-c)^2$. Thus, without any assumption, we can get
\begin{align*}
\EE[G_{t+1}\mathds{1}_{\cA_t}|\cF_t]\leq G_{t}\mathds{1}_{\cA_t}\leq G_{t}\mathds{1}_{\cA_{t-1}}.
\end{align*}

This means $\{G_{t}\mathds{1}_{\cA_{t-1}}\}$ is a supermartingale. To use Azuma's inequality, we need to bound the following difference
\begin{align*}
\tilde{d}_{t+1}&\triangleq|G_{t+1}\mathds{1}_{\cA_t}-\EE[G_{t+1}\mathds{1}_{\cA_t}|\cF_t]|\\
&=(1-4\eta c)^{-t-1}|(\alpha_{1,t+1}\alpha_{2,t+1}-1)^2-\EE[(\alpha_{1,t+1}\alpha_{2,t+1}-1)^2|\cF_t]|\mathds{1}_{\cA_t}\\
&\leq(1-4\eta c)^{-t-1}\tilde{C}_3,
\end{align*}
where {$\tilde{C}_3=O(\sqrt{\gamma}\tilde{C}_2)$}. Further define $\tilde{r}_t\triangleq\sqrt{\sum_{i=1}^t \tilde{d_i}^2}$, by Azuma's inequality we have
\begin{align*}
\mathbb{P}\left(G_{t} \mathds{1}_{\cA_{t-1}}-G_{0} \geq O(1) \tilde{r}_{t} \log ^{\frac{1}{2}}\left(\frac{1}{\eta^{2} \delta}\right)\right) \leq \exp \left(-\frac{O(1) \tilde{r}_{t}^{2} \log \left(\frac{1}{\eta^{2} \delta}\right)}{2 \sum_{i=0}^{t} \tilde{d_{i}}^{2}}\right)=O\left(\eta^{2} \delta\right).
\end{align*}
Thus, with probability at least $1-O(\eta^2 \delta)$,
\begin{align*}
((\alpha_{1,t}\alpha_{2,t}-1)^2-\gamma) \mathds{1}_{\cA_{t-1}} &< (1-4\eta c)^t\left((\alpha_{1,0}\alpha_{2,0}-1)^2+O(1) \tilde{r}_{t} \log ^{\frac{1}{2}}\left(\frac{1}{\eta^{2} \delta}\right)\right)\\
&<(\alpha_{1,0}\alpha_{2,0}-1)^2+O\left( (1-4\eta c)^t \tilde{r}_{t} \log ^{\frac{1}{2}}\left(\frac{1}{\eta^{2} \delta}\right)\right)\\
&<4\gamma+O\left( {\frac{\sqrt{\gamma\eta}}{\sigma^2}\left(\sqrt{\log\frac{d_1+d_2}{\eta^2}}+\sqrt{\log\frac{1}{\delta}}\right)}\log ^{\frac{1}{2}}\left(\frac{1}{\eta^{2} \delta}\right)\right),
\end{align*}
When $\cE_{t-1}$ holds, w.p. at least $1-O(\eta^2 \delta)$
\begin{align*}
(\alpha_{1,t}\alpha_{2,t}-1)^2&<5\gamma+O\left( \frac{{\tilde{C}_3}}{\sqrt{\eta}} \log ^{\frac{1}{2}}\left(\frac{1}{\eta^{2} \delta}\right)\right)
<6\gamma,
\end{align*}
by choosing {$\eta\leq\tilde{\eta}_3=O\left(\sigma^6(\log\frac{1}{\delta})^{-1}\left(\log\frac{d_1+d_2}{\delta}\right)^{-1}\right)$} we can guarantee the last term is smaller than $\gamma$. Now we know that if $\cA_{t-1}$ holds, $\cA_{t}$ holds w.p. at least $1-O(\eta^2 \delta)$. It is easy to show that the Perturbed GD updates satisfy $(\alpha_{1,t}\alpha_{2,t}-1)^2<6\gamma$ in the following $O(\frac{1}{\eta^2})$ steps w.p. at least $1-\delta$. 
\end{proof}
\subsection{Proof of Lemma \ref{lem_convergence}}
\begin{lemma} \label{a1=a2}
Suppose $(x_t^\top My_t-1)^2<6\gamma$ holds for all t, where $\gamma$ is as defined above. For any $\delta \in (0,1)$ and any $\Delta>0$, if we choose $\sigma=O\left((\log\frac{1}{\delta})^{-\frac{1}{3}}\right)$ and take step size $$\eta\leq\tilde{\eta}_4=O\left(\sigma^{10}\Delta\right),$$
then with probability at least $1-\delta$, we have \begin{align*}
    (x_t,y_t) \in \left\{(x,y)| \left((x^\top u_*)^2-(y^\top v_*)^2\right)^2<6\Delta \right\},
\end{align*} for all t's such that $\tau_4\leq t \leq T_1$, where $T_1\triangleq O(\frac{1}{\eta^2})$ and
$\tau_4\triangleq O(\frac{1}{\eta\sigma^2}\log\frac{1}{\eta}\log\frac{1}{\delta})$.
\end{lemma}

Again, we partition this lemma into two parts. It is easy to prove Lemma \ref{a1=a2} from Lemmas \ref{lma3} and \ref{lma4}.

\begin{lemma} \label{lma3}
Suppose $\forall t\leq T_1= O(\frac{1}{\eta^2})$,
\begin{align*}
    (x_t,y_t) \in \left\{(x,y)| (x^\top My-1)^2<6\gamma\right\}.
\end{align*}then $\forall \delta \in (0,1)$, with probability at least $1-\delta$, there exists a time step $\tau\leq \tau_4= O(\frac{1}{\eta\sigma^2}\log\frac{1}{\eta}\log\frac{1}{\delta})$ such that 
\begin{align*}
    \left((x_{\tau}^\top u_*)^2-(y_{\tau}^\top v_*)^2\right)^2<4\Delta,
\end{align*}
where {$\Delta=O(\frac{\eta}{\sigma^{10}})$}.
\end{lemma}
\begin{proof}
\begin{align*}
    \EE_\xi[D_t]&=\EE_\xi[\alpha_{1,t}<\nabla_{x}\cF(x_{t}+\xi_{1,t},y_{t}+\xi_{2,t}),u_*>-\alpha_{2,t}<\nabla_{y}\cF(x_{t}+\xi_{1,t},y_{t}+\xi_{2,t}),v_*>]
    \\&=\alpha_{1,t}(\alpha_{1,t}(\alpha_{2,t}^2+\norm{\beta_{2,t}}_2^2)-\alpha_{2,t}+\alpha_{1,t}\sigma^2)-\alpha_{2,t}(\alpha_{2,t}(\alpha_{1,t}^2+\norm{\beta_{1,t}}_2^2)-\alpha_{1,t}+\alpha_{2,t}\sigma^2)
    \\&=\sigma^2(\alpha_{1,t}^2-\alpha_{2,t}^2)+(\alpha_{1,t}^2\norm{\beta_{2,t}}_2^2-\alpha_{2,t}^2\norm{\beta_{1,t}}_2^2).
\end{align*}
Thus, we have
\begin{align*}
    \EE_\xi[D_t](\alpha_{1,t}^2-\alpha_{2,t}^2)&=\sigma^2(\alpha_{1,t}^2-\alpha_{2,t}^2)^2+(\alpha_{1,t}^2-\alpha_{2,t}^2)(\alpha_{1,t}^2\norm{\beta_{2,t}}_2^2-\alpha_{2,t}^2\norm{\beta_{1,t}}_2^2)
    \\&\geq\sigma^2(\alpha_{1,t}^2-\alpha_{2,t}^2)^2-2\alpha_{1,t}^2\alpha_{2,t}^2\frac{2\eta C_2}{\sigma^2} \\&>\sigma^2(\alpha_{1,t}^2-\alpha_{2,t}^2)^2-2(1+\sqrt{6\gamma})^2\frac{2\eta C_2}{\sigma^2}
    \\&>\sigma^2(\alpha_{1,t}^2-\alpha_{2,t}^2)^2-3\frac{2\eta C_2}{\sigma^2}.
\end{align*}
The last inequality can be achieved by choosing $\gamma<\frac{(\sqrt{3/2}-1)^2}{6}$. Plugging \eqref{*4} and \eqref{*5} into \eqref{*3}, taking expectation conditioning on previous trajectory $\cF_t$ and plugging the equation above in, we get
\begin{align*}
\EE[(\alpha_{1,t+1}^2-\alpha_{2,t+1}^2)^2|\cF_t]
\nonumber&=(\alpha_{1,t}^2-\alpha_{2,t}^2)^2-4\eta \EE_\xi[D_t](\alpha_{1,t}^2-\alpha_{2,t}^2)\\&+4\eta^2 \EE_\xi[D_t^2]+\eta^4 \EE_\xi[F_t^2]-4\eta^3\EE_\xi[D_t F_t]+2\eta^2\EE_\xi[F_t]\left(\alpha_{1,t}^2-\alpha_{2,t}^2\right)\\
&\leq\left(1-4\eta \sigma^2\right)(\alpha_{1,t+1}^2-\alpha_{2,t+1}^2)^2+24\eta \frac{\eta C_2}{\sigma^2}+\tilde{C}_4\eta^2,
\end{align*}
where ${\tilde{C}_4=O(\frac{1}{\sigma^4}\left(\sqrt{\log\frac{d_1+d_2}{\eta^2}}+\sqrt{\log\frac{1}{\delta}}\right)^2)}$. Denote $\Delta\triangleq\frac{24\eta \frac{\eta C_2}{\sigma^2}+\tilde{C}_4\eta^2}{4\eta \sigma^2}$. the inequality above can be re-expressed as
\begin{align*}
\EE[\{(\alpha_{1,t+1}^2-\alpha_{2,t+1}^2)^2-\Delta\}|\cF_t]\leq\left(1-4\eta \sigma^2\right)\{(\alpha_{1,t}^2-\alpha_{2,t}^2)^2-\Delta\}.
\end{align*}

Denote $\cB_t\triangleq\{\forall \tau\leq t : ((\alpha_{1,\tau}^2-\alpha_{2,\tau}^2)^2\geq4\Delta\}$, for $t>\frac{\log(\frac{\left(\alpha_{1,0}^2-\alpha_{2,0}^2\right)^2}{\delta})}{4\eta \sigma^2}$, we have
\begin{align*}
    4\Delta\EE[\mathds{1}_{\cB_t}]\leq\EE[(\alpha_{1,t}^2-\alpha_{2,t}^2)^2-\Delta]\leq(1-4\eta c)^t((\alpha_{1,0}^2-\alpha_{2,0}^2)^2-\Delta)+\Delta<2\Delta,
\end{align*}
where the first inequality comes from the definition of $\cB_t$ and the second one comes from calculation above. Thus, if we choose $t=O(\frac{\log(\frac{1}{\Delta \sigma^2})}{\eta \sigma^2})$ and recursively applying the inequality above $\log(\frac{1}{\delta})$ times,
we will get, for $\tau_4=O(\frac{1}{\eta \sigma^2}\log(\frac{1}{\delta})\log(\frac{1}{\Delta \sigma^2}))=O(\frac{1}{\eta \sigma^2}\log(\frac{1}{\delta})\log(\frac{1}{\eta}))$,
\begin{align*}
    \PP(\cB_{\tau_4})<(\frac{1}{2})^{\log(\frac{1}{\delta})}=\delta.
\end{align*}
Thus, w.p. at least $1-\delta$, there exists a $\tau\leq \tau_4$ s.t. $(\alpha_{1,\tau}^2-\alpha_{2,\tau}^2)^2<4\Delta$.  Here, we finish the proof of Lemma \ref{lma3}. 	
\end{proof}

\begin{lemma} \label{lma4}
Suppose there exists a time step $\tau\leq \tau_4= O(\frac{1}{\eta\sigma^2}\log\frac{1}{\Delta\sigma^2}\log\frac{1}{\delta})$ such that 
\begin{align*}
    \left((x_{\tau}^\top u_*)^2-(y_{\tau}^\top v_*)^2\right)^2<4\Delta,
\end{align*}then $\forall \delta \in (0,1)$, with probability at least $1-\delta$, $\forall t\leq T_1\triangleq O(\frac{1}{\eta^2})$,
\begin{align*}
    (x_t,y_t) \in \left\{(x,y)| \left((x^\top u_*)^2-(y^\top v_*)^2\right)^2<6\Delta \right\}.
\end{align*}
\end{lemma}


\begin{proof}

Without loss of generality, we place the time origin at $\tau$, i.e. $(\alpha_{1,0}^2-\alpha_{2,0}^2)^2<4\Delta$. 
Denote $\cD\triangleq\{(\alpha_1,\alpha_2)|(\alpha_1^2-\alpha_2^2)^2<6\Delta\}$ and $\cD_t\triangleq\{\forall \tau\leq t:(\alpha_{1,\tau},\alpha_{2,\tau})\in \cD\}$. Note that $(\alpha_{1,t}\alpha_{2,t}-1)^2<6\gamma$ still holds with high probability. Defining $H_t\triangleq(1-4\eta\sigma^2)^{-t}\{(\alpha_{1,t}^2-\alpha_{2,t}^2)^2-\Delta\}$, we can get
\begin{align*}
\EE[H_{t+1}\mathds{1}_{\cD_t}|\cF_t]\leq H_{t}\mathds{1}_{\cD_t}\leq H_{t}\mathds{1}_{\cD_{t-1}}.
\end{align*}

This means $\{H_{t}\mathds{1}_{\cD_{t-1}}\}$ is a supermartingale. To use Azuma's inequality, we need to bound the following difference
\begin{align*}
\Bar{d}_{t+1}&\triangleq|H_{t+1}\mathds{1}_{\cD_t}-\EE[H_{t+1}\mathds{1}_{\cD_t}|\cF_t]|\\
&=(1-4\eta \sigma^2)^{-t-1}|(\alpha_{1,t+1}^2-\alpha_{2,t+1}^2)^2-\EE[(\alpha_{1,t+1}^2-\alpha_{2,t+1}^2)^2|\cF_t]|\mathds{1}_{\cD_t}\\
&\leq(1-4\eta \sigma^2)^{-t-1}\tilde{C}_5,
\end{align*}
where ${\tilde{C}_5=O(\eta\frac{\sqrt{\Delta}}{\sigma^2}\left(\sqrt{\log\frac{d_1+d_2}{\eta^2}}+\sqrt{\log\frac{1}{\delta}}\right))}$. Further define $\Bar{r}_t\triangleq\sqrt{\sum_{i=1}^t \Bar{d_i}^2}$, by Azuma's inequality we will get
\begin{align*}
\mathbb{P}\left(H_{t} \mathds{1}_{\cD_{t-1}}-H_{0} \geq O(1) \Bar{r}_{t} \log ^{\frac{1}{2}}\left(\frac{1}{\eta^{2} \delta}\right)\right) \leq \exp \left(-\frac{O(1) \Bar{r}_{t}^{2} \log \left(\frac{1}{\eta^{2} \delta}\right)}{2 \sum_{i=0}^{t} \Bar{d_{i}}^{2}}\right)=O\left(\eta^{2} \delta\right).
\end{align*}
Thus, with probability at least $1-O(\eta^2 \delta)$, we have
\begin{align*}
((\alpha_{1,t}^2-\alpha_{2,t}^2)^2-\Delta) \mathds{1}_{\cD_{t-1}} &< (1-4\eta \sigma^2)^t\left((\alpha_{1,0}^2-\alpha_{2,0}^2)^2+O(1) \Bar{r}_{t} \log ^{\frac{1}{2}}\left(\frac{1}{\eta^{2} \delta}\right)\right)\\
&<(\alpha_{1,0}^2-\alpha_{2,0}^2)^2+O\left( (1-4\eta \sigma^2)^t \Bar{r}_{t} \log ^{\frac{1}{2}}\left(\frac{1}{\eta^{2} \delta}\right)\right)\\
&<4\Delta+O\left( \frac{\sqrt{\Delta\eta}}{\sigma^2}\left(\sqrt{\log\frac{d_1+d_2}{\eta^2}}+\sqrt{\log\frac{1}{\delta}}\right) \log ^{\frac{1}{2}}\left(\frac{1}{\eta^{2} \delta}\right)\right).
\end{align*}
When $\cD_{t-1}$ holds, w.p. at least $1-O(\eta^2 \delta)$, by the inequality above we have
\begin{align*}
(\alpha_{1,t}^2-\alpha_{2,t}^2)^2&<5\Delta+O\left( \frac{\sqrt{\Delta\eta}}{\sigma^2}\left(\sqrt{\log\frac{d_1+d_2}{\eta^2}}+\sqrt{\log\frac{1}{\delta}}\right) \log ^{\frac{1}{2}}\left(\frac{1}{\eta^{2} \delta}\right)\right)
<6\Delta.
\end{align*}
Note that to make sure last terms is smaller than $\Delta$, we need 
\begin{align*}
     \frac{\sqrt{\eta}}{\sigma^2\sqrt{\Delta}}\left(\sqrt{\log\frac{d_1+d_2}{\eta^2}}+\sqrt{\log\frac{1}{\delta}}\right) \log ^{\frac{1}{2}}\left(\frac{1}{\eta^{2} \delta}\right)=O(1).
\end{align*}
As $\Delta=O(\frac{\eta}{\sigma^{10}})$, we know that as long as $\eta$ is polynomial in $\sigma$, choosing $\sigma=O((\log\frac{1}{\delta})^{-\frac{1}{3}})$ is sufficient. Now we know that if $\cD_{t-1}$ holds, $\cD_{t}$ holds w.p. at least $1-O(\eta^2 \delta)$. It is easy to show that the Perturbed GD updates satisfy $(\alpha_{1,t}^2-\alpha_{2,t}^2)^2<6\Delta$ in the following $O(\frac{1}{\eta^2})$ steps w.p. at least $1-\delta$. 
\end{proof}

With Lemmas \ref{lem_bounded}, \ref{lem_loss} and \ref{a1=a2}, we can prove Lemma \ref{lem_convergence}. Here follows a brief proof.
\begin{proof}
\begin{align*}
    |1-x_t^\top u_*|&<(1+x_t^\top u_*)|1-x_t^\top u_*|\\
    &=|1-(x_t^\top u_*)^2+x_t^\top u_*v_*^\top y_t-x_t^\top u_*v_*^\top y_t|\\
    &\leq|1-x_t^\top u_*v_*^\top y_t|+|(x_t^\top u_*)^2-x_t^\top u_*v_*^\top y_t|\\
    &=|1-x_t^\top M y_t|+x_t^\top u_*|x_t^\top u_*-v_*^\top y_t|\\
    &\leq|1-x_t^\top M y_t|+\frac{\sqrt{2}}{\sigma}|x_t^\top u_*-v_*^\top y_t|\\
    &<\sqrt{6\gamma}+\frac{\sqrt{2}}{\sigma}\sqrt{6\Delta}.
\end{align*}
The last inequality comes from Lemmas \ref{lma22} and \ref{lma4}. Together with Lemma \ref{lem_bounded} we can get \begin{align*}
    \norm{x_t-u_*}^2=(1-x_t^\top u_*)^2+\norm{\beta_{1,t}}_2^2<(\sqrt{6\gamma}+\frac{\sqrt{2}}{\sigma}\sqrt{6\Delta})^2+2\frac{\eta C_2}{\sigma^2}=O(\sigma^2+\frac{\eta}{\sigma^{10}}).
\end{align*}
Note that we use $C_2=O(\frac{1}{\sigma^6})$ when calculating the order. Similarly, we have
\begin{align*}
    \norm{y_t-v_*}^2<(\sqrt{6\gamma}+\frac{\sqrt{2}}{\sigma}\sqrt{6\Delta})^2+2\frac{\eta C_2}{\sigma^2}=O(\sigma^2+\frac{\eta}{\sigma^{10}}).
\end{align*}
Then, for any $\epsilon>0$, by choosing $\sigma=O(\sqrt{\epsilon})$ and $\eta\leq\eta_5=O(\sigma^{10} \epsilon)$, we will have $\norm{x_t-u_*}^2<\epsilon$ and $\norm{y_t-v_*}^2<\epsilon$
\end{proof}

\section{Proof of Theorem \ref{thm_rankr}}\label{pf_3}
Recall that the gradient of $\tilde\cF$ takes the following form.
\begin{align*}
\nabla_X\tilde\cF(X,Y)&=(XY^\top-M)Y-d_2\sigma_2^2X,\\
\nabla_Y\tilde\cF(X,Y)&=(XY^\top-M)^\top X-d_1\sigma_1^2Y.
\end{align*}
Suppose $(U,V)$ is a stationary point. Then we have
\begin{align}
(UV^\top-M)V-d_2\sigma_2^2U&=0,\label{eq_stat_1}\\
(UV^\top-M)^\top U-d_1\sigma_1^2V&=0. \label{eq_stat_2}
\end{align}
 \noindent $\bullet$ {\bf Step 1:} To prove the first statement,  simply left multiply each side of  \eqref{eq_stat_1} by $U^\top$ and each side of  \eqref{eq_stat_2} by $V^\top,$ and we have the following equations.
\begin{align*}
U^\top UV^\top V-U^\top MV-d_2\sigma_2^2U^\top U&=0,\\
V^\top VU^\top U-V^\top M^\top U-d_1\sigma_1^2V^\top V&=0. 
\end{align*}
Note that the following equation naturally holds.
$$U^\top UV^\top V-U^\top MV=\left(V^\top VU^\top U-V^\top M^\top U\right)^\top.$$
Combine these three equations together and we have
$$U^\top U=\frac{d_1\sigma_1^2}{d_2\sigma_2^2}(V^\top V)^\top=\gamma^2V^\top V.$$
 \noindent $\bullet$ {\bf Step 2:} We next show that  $(\tilde U,\tilde V)R$  is a stationary point, where $$(\tilde U,\tilde V)=(\sqrt{\gamma} A(\Sigma-\gamma\sigma^2I_r)^{\frac{1}{2}},\frac{1}{\sqrt{\gamma} }B(\Sigma-\gamma\sigma^2I_r)^{\frac{1}{2}}),$$
where $R\in \RR^{r\times r}$ is an orthogonal matrix. We only need to check \eqref{eq_stat_1} and  \eqref{eq_stat_2}.  In fact, we have 
\begin{align*}
\nabla_X\tilde\cF(\tilde UR,\tilde VR)&=-\gamma\sigma^2 AB^\top\frac{1}{\sqrt{\gamma} }B(\Sigma-\gamma\sigma^2I_r)^{\frac{1}{2}}R+\sigma^2\sqrt{\gamma} A(\Sigma-\gamma\sigma^2I_r)^{\frac{1}{2}}R\\
&=-\sigma^2\sqrt{\gamma} A(\Sigma-\gamma\sigma^2I_r)^{\frac{1}{2}}R+\sigma^2\sqrt{\gamma} A(\Sigma-\gamma\sigma^2I_r)^{\frac{1}{2}}R\\
&=0,
\end{align*}
and
\begin{align*}
\nabla_Y\tilde\cF(\tilde UR,\tilde VR)&=-\gamma\sigma^2 BA^\top\sqrt{\gamma} A(\Sigma-\gamma\sigma^2I_r)^{\frac{1}{2}}R+\sigma^2\gamma^2\frac{1}{\sqrt{\gamma} }B(\Sigma-\gamma\sigma^2I_r)^{\frac{1}{2}}R\\
&=-\sigma^2\gamma\sqrt{\gamma} B(\Sigma-\gamma\sigma^2I_r)^{\frac{1}{2}}R+\sigma^2\gamma\sqrt{\gamma} B(\Sigma-\gamma\sigma^2I_r)^{\frac{1}{2}}R\\
&=0.
\end{align*}
Combine the above equations together and we know that $(\tilde U,\tilde V)$ is a stationary point.

 \noindent $\bullet$ {\bf Step 3:}  We next show that $\{(\tilde U,\tilde V)R\big| R\in\RR^{r\times r}, \text{orthogonal}\}$ are the global minima and all other stationary points enjoy strict saddle property. Without loss of generality, we assume $\gamma=1.$
 
 We first calculate the Hessian $\nabla^2 \tilde\cF(X,Y).$ The Hessian can be viewed as a matrix that operates on vectorized matrices of dimension $(d_1+d_2)\times r.$ Then, for any $W\in\RR^{(d_1+d_2)\times r}$, the Hessian defines a quadratic form
 $$[\nabla^2 \tilde\cF(W)](Z_1,Z_2)=\sum_{i,j,k,l}\frac{\partial^2\tilde\cF(W)}{\partial W[i,j]\partial W[k,l]}Z_1[i,j]Z_2[k,l], ~\forall Z_1,Z_2\in\RR^{(d_1+d_2)\times r}.$$
 We can then express the Hessian $\nabla^2 \tilde\cF(W)$ as follows:
 \begin{align*}
 [\nabla^2 \tilde\cF(X,Y)](\Delta,\Delta)=2<XY^\top-M,\Delta_U\Delta_V^\top>+\norm{U\Delta_V^\top+\Delta_U V^\top}_{\rm{F}}^2+\sigma^2 \norm{\Delta_U}_{\rm{F}}^2+\sigma^2 \norm{\Delta_V}_{\rm{F}}^2,
 \end{align*}
 where $\Delta=\begin{bmatrix} 
\Delta_U\\
\Delta_V
\end{bmatrix},$ $\Delta_U\in \RR^{d_1\times r}$ and $\Delta_V\in \RR^{d_2\times r}.$  
We further denote $W=\begin{bmatrix} 
X\\
Y
\end{bmatrix},$ $\tilde W=\begin{bmatrix} 
\tilde U\\
\tilde V
\end{bmatrix},$  and $\tilde M=\tilde U\tilde V^\top,$
$$R=\argmin_{R'\in\RR^{r\times r}, \text{orthogonal}}\norm{W-\tilde WR'}.$$ We then have the following lemma.
\begin{lemma}\label{lem_negative_curv}
Let $\sigma_{\min}(M)$ be the smallest singular value of $M$. Suppose $d_1\sigma_1^2=d_2 \sigma_2^2=\sigma^2,$ and  $\sigma^2<\sigma_{\min}(M).$ For $\forall(U,V)\in (\RR^{d_1},\RR^{d_2})$ such that $\nabla\tilde\cF(U,V)=0,$ we denote $\Delta= \begin{bmatrix} 
U-\tilde U R\\
V-\tilde V R
\end{bmatrix},$ then we have
\begin{align}\label{negative_curv}
[\nabla^2 \tilde\cF(U,V)](\Delta,\Delta)\leq -\norm{UV^\top-\tilde M}_{\rm{F}}^2-3\sigma^2\norm{A^\top U-B^\top V}_{\rm{F}}^2.
\end{align}
Moreover, $[\nabla^2 \tilde\cF(U,V)](\Delta,\Delta)<0$ if 
$$(U,V)\notin \{(\tilde U,\tilde V)R'\big|R'\in \RR^{r\times r}, R'R'^\top=R'^\top R'=I_r \}.$$\end{lemma}
\begin{proof}
Recall that the quadratic form defined by the Hessian can be written as follows.
\begin{align}\label{quadratic}
 [\nabla^2 \tilde\cF(U,V)](\Delta,\Delta)=2<UV^\top-M,\Delta_U\Delta_V^\top>+\norm{U\Delta_V^\top+\Delta_U V^\top}_{\rm{F}}^2+\sigma^2 \norm{\Delta_U}_{\rm{F}}^2+\sigma^2 \norm{\Delta_V}_{\rm{F}}^2.
 \end{align}
 We start from the second term $\norm{U\Delta_V^\top+\Delta_U V^\top}_{\rm{F}}^2$. Similar to the proof of Claim B.5 in \citet{du2018algorithmic}, we have
 \begin{align*}
 \norm{U\Delta_V^\top+\Delta_U V^\top}_{\rm{F}}^2&= \norm{\Delta_U\Delta_V^\top+ UV^\top-\tilde M}_{\rm{F}}^2\\
 &= \norm{\Delta_U\Delta_V^\top}_{\rm{F}}^2+ \norm{UV^\top-\tilde M}_{\rm{F}}^2+2 <\Delta_U\Delta_V^\top, UV^\top-\tilde M>\\
 &= \norm{\Delta_U\Delta_V^\top}_{\rm{F}}^2+ \norm{UV^\top-\tilde M}_{\rm{F}}^2+2 <\Delta_U\Delta_V^\top, UV^\top-M>+2 <\Delta_U\Delta_V^\top, M-\tilde M>\\
 &= \norm{\Delta_U\Delta_V^\top}_{\rm{F}}^2+ \norm{UV^\top-\tilde M}_{\rm{F}}^2+2 <\Delta_U\Delta_V^\top, UV^\top-M>+2 \sigma^2<\Delta_U\Delta_V^\top, AB^\top>
 \end{align*} 
 Plugging this equation into \eqref{quadratic}, we have
 \begin{align*} [\nabla^2 \tilde\cF(U,V)](\Delta,\Delta)&=4<UV^\top-M,\Delta_U\Delta_V^\top>+ \norm{\Delta_U\Delta_V^\top}_{\rm{F}}^2+ \norm{UV^\top-\tilde M}_{\rm{F}}^2\\&~~~+\sigma^2 \left(\norm{\Delta_U}_{\rm{F}}^2+ \norm{\Delta_V}_{\rm{F}}^2+2 <\Delta_U\Delta_V^\top, AB^\top>
\right).
 \end{align*}
 Note that using the fact $\nabla\tilde\cF(U,V)=\nabla\tilde\cF(\tilde U,\tilde V)=0,$ one can easily verify that 
\begin{align*}
4<UV^\top-M,\Delta_U\Delta_V^\top>&=4<UV^\top-M,\tilde M>-2\sigma^2 \left(\norm{\Delta_U}_{\rm{F}}^2+ \norm{\Delta_V}_{\rm{F}}^2\right)+2\sigma^2 \left(\norm{\tilde U}_{\rm{F}}^2+ \norm{\tilde V}_{\rm{F}}^2\right)\\
&=-4\norm{UV^\top-\tilde M}_{\rm{F}}^2-4\sigma^2<AB^\top,\tilde M-UV^\top>\\&~~~+2\sigma^2\left(\norm{\tilde U}_{\rm{F}}^2+ \norm{\tilde V}_{\rm{F}}^2-\norm{ U}_{\rm{F}}^2-\norm{ V}_{\rm{F}}^2\right).
\end{align*}
Thus, 
\begin{align} 
[\nabla^2 \tilde\cF(U,V)](\Delta,\Delta)&=-4\norm{UV^\top-\tilde M}_{\rm{F}}^2+ \norm{\Delta_U\Delta_V^\top}_{\rm{F}}^2+ \norm{UV^\top-\tilde M}_{\rm{F}}^2\nonumber\\
&~~~-\sigma^2 \Bigg[\norm{\Delta_U}_{\rm{F}}^2+ \norm{\Delta_V}_{\rm{F}}^2-2 <\Delta_U\Delta_V^\top, AB^\top>\nonumber\\
&~~~~~-2\left(\norm{\tilde U}_{\rm{F}}^2+ \norm{\tilde V}_{\rm{F}}^2-\norm{ U}_{\rm{F}}^2-\norm{ V}_{\rm{F}}^2\right)
+4<AB^\top,\tilde M-UV^\top>\Bigg].\label{eq_quad}
 \end{align}
We then have the following two claims:
\begin{claim}\label{claim_a}
$-4\norm{UV^\top-\tilde M}_{\rm{F}}^2+ \norm{\Delta_U\Delta_V^\top}_{\rm{F}}^2+ \norm{UV^\top-\tilde M}_{\rm{F}}^2\leq -\norm{UV^\top-\tilde M}_{\rm{F}}^2.$
\end{claim}
\begin{proof}
Similar to the proof of Claim B.5 in \citet{du2018algorithmic}, we have
\begin{align*}
 \norm{\Delta_U\Delta_V^\top}_{\rm{F}}^2&\leq\frac{1}{4}\norm{\Delta\Delta^\top}_{\rm{F}}^2\\
 &\leq \frac{1}{2}\norm{WW^\top-\tilde W\tilde W^\top}_{\rm{F}}^2\\
 &=2\norm{UV^\top-\tilde M}_{\rm{F}}^2-\norm{U^\top\tilde U-V^\top\tilde V}_{\rm{F}}^2+\frac{1}{2}\norm{U^\top U-V^\top V}_{\rm{F}}^2+\frac{1}{2}\norm{\tilde U^\top  \tilde U-\tilde V^\top \tilde V}_{\rm{F}}^2\\
 &=2\norm{UV^\top-\tilde M}_{\rm{F}}^2-\norm{U^\top\tilde U-V^\top\tilde V}_{\rm{F}}^2\\
 &\leq 2\norm{UV^\top-\tilde M}_{\rm{F}}^2.
\end{align*}
Thus,
\begin{align*}
-4\norm{UV^\top-\tilde M}_{\rm{F}}^2+ \norm{\Delta_U\Delta_V^\top}_{\rm{F}}^2+ \norm{UV^\top-\tilde M}_{\rm{F}}^2&\leq-4\norm{UV^\top-\tilde M}_{\rm{F}}^2+  2\norm{UV^\top-\tilde M}_{\rm{F}}^2+ \norm{UV^\top-\tilde M}_{\rm{F}}^2\\&=-\norm{UV^\top-\tilde M}_{\rm{F}}^2.
\end{align*}
\end{proof}
\begin{claim}\label{claim_b}
$\norm{ U}_{\rm{F}}^2+\norm{ V}_{\rm{F}}^2-\norm{\tilde U}_{\rm{F}}^2- \norm{\tilde V}_{\rm{F}}^2
+2<AB^\top,\tilde M-UV^\top>=\norm{A^\top \Delta_U-B^\top \Delta_V}_{\rm{F}}^2 .$
\end{claim}
\begin{proof}
First, the LHS of the equation can be rewritten as follows.
\begin{align*}
&\norm{ U}_{\rm{F}}^2+\norm{ V}_{\rm{F}}^2-\norm{\tilde U}_{\rm{F}}^2- \norm{\tilde V}_{\rm{F}}^2
+2<AB^\top,\tilde M-UV^\top>\\
=&\norm{ U}_{\rm{F}}^2+\norm{ V}_{\rm{F}}^2-2<AB^\top,\tilde M+UV^\top>+4<AB^\top,\tilde M>-\norm{\tilde U}_{\rm{F}}^2- \norm{\tilde V}_{\rm{F}}^2\\
=&\norm{ U}_{\rm{F}}^2+\norm{ V}_{\rm{F}}^2-2<AB^\top,\tilde M+UV^\top>+4\norm{\tilde U}_{\rm{F}}^2-\norm{\tilde U}_{\rm{F}}^2- \norm{\tilde V}_{\rm{F}}^2\\
=&\norm{ U}_{\rm{F}}^2+\norm{ V}_{\rm{F}}^2 +\norm{\tilde U}_{\rm{F}}^2+\norm{\tilde V}_{\rm{F}}^2-2<AB^\top,\tilde M+UV^\top>,
\end{align*}
where we use the fact 
$$<AB^\top,\tilde M>=\norm{\tilde U}_{\rm{F}}^2=\norm{\tilde V}_{\rm{F}}^2.$$
On the other hand, the RHS of the equation can be rewritten as follows.
\begin{align*}
\norm{A^\top \Delta_U-B^\top \Delta_V}_{\rm{F}}^2&=\norm{\Delta_U}_{\rm{F}}^2+ \norm{\Delta_V}_{\rm{F}}^2-2 <\Delta_U\Delta_V^\top, AB^\top>\\
&=\norm{ U}_{\rm{F}}^2+\norm{ V}_{\rm{F}}^2 +\norm{\tilde U}_{\rm{F}}^2+\norm{\tilde V}_{\rm{F}}^2-2<AB^\top,\tilde M+UV^\top>\\
&~~~~~-2\tr(U(\tilde UR)^\top)-2\tr(V(\tilde VR)^\top)+2\tr(U(\tilde UR)^\top)+2\tr(V(\tilde VR)^\top)\\
&=\norm{ U}_{\rm{F}}^2+\norm{ V}_{\rm{F}}^2 +\norm{\tilde U}_{\rm{F}}^2+\norm{\tilde V}_{\rm{F}}^2-2<AB^\top,\tilde M+UV^\top>\\
&=\norm{ U}_{\rm{F}}^2+\norm{ V}_{\rm{F}}^2-\norm{\tilde U}_{\rm{F}}^2- \norm{\tilde V}_{\rm{F}}^2
+2<AB^\top,\tilde M-UV^\top>.
\end{align*}
\end{proof}
Plugging the conclusions in Claims \ref{claim_a} and \ref{claim_b} into \eqref{eq_quad}, we have
\begin{align*} 
[\nabla^2 \tilde\cF(U,V)](\Delta,\Delta)\leq-\norm{UV^\top-\tilde M}_{\rm{F}}^2-3\sigma^2\norm{A^\top \Delta_U-B^\top \Delta_V}_{\rm{F}}^2.
\end{align*}
Note that since $A^\top \tilde U R=B^\top \tilde V R,$ we have $A^\top \Delta_U-B^\top \Delta_V=A^\top U-B^\top V.$ 
To justify our last statement, we have the following claim.
\begin{claim}\label{claim_v}
$\norm{UV^\top-\tilde M}_{\rm{F}}^2+3\sigma^2\norm{A^\top U-B^\top V}_{\rm{F}}^2=0$ if and only if 
$$(U,V)=(\tilde U,\tilde V)R,$$
where $R\in \RR^{r\times r}$ is an orthogonal matrix.
\end{claim}
\begin{proof}
$\norm{UV^\top-\tilde M}_{\rm{F}}^2+3\sigma^2\norm{A^\top U-B^\top V}_{\rm{F}}^2=0$ if and only if 
\begin{align}
UV^\top&=\tilde M,\label{con1}\\
A^\top U&=B^\top V.\label{con2}
\end{align}
Left multiplying each side of \eqref{con1} by $A^\top,$  we have
\begin{align*}
A^\top UV^\top=(\Sigma-\sigma^2I)B^\top
&\Leftrightarrow B^\top VV^\top=(\Sigma-\sigma^2I)B^\top\\
&\Leftrightarrow  VV^\top=B(\Sigma-\sigma^2I)B^\top\\
&\Leftrightarrow  V=B(\Sigma-\sigma^2I)^{\frac{1}{2}}R=\tilde VR' .
\end{align*}
The last equivalent argument comes from Theorem 4 in \citet{li2019symmetry}. Plugging $V=\tilde VR'$ back to \eqref{con2}, we have $U=\tilde UR'.$
\end{proof}
As a direct result of  Claim \ref{claim_v}, for stationary point $(U,V)\notin\{(\tilde U,\tilde V)R'\big|R'\in \RR^{r\times r}, R'R'^\top=R'^\top R'=I_r \},$ we have
$$[\nabla^2 \tilde\cF(U,V)](\Delta,\Delta)<0.$$
We prove the lemma.
\end{proof}
Lemma \ref{lem_negative_curv} directly implies that $\{(\tilde U,\tilde V)R'\big|R'\in \RR^{r\times r}, R'R'^\top=R'^\top R'=I_r \}$ contains all the global optima, and all other stationary points enjoy strict saddle property.

\newpage
\section{Perturbed GD}\label{alg_detail}
The detail of the Perturbed GD algorithm is summarized in Algorithm \ref{alg:Perturbed GD}.
\begin{algorithm}[H]
    \caption{Perturbed Gradient Descent for Rank-1 Matrix Factorization.}
    \label{alg:Perturbed GD}
    \begin{algorithmic}
    	\STATE{\textbf{Input}: step size $\eta$, noise level $\sigma_1, \sigma_2$, matrix $M \in \RR^{d_1 \times d_2}$, number of iterations $T$.}
	\STATE{\textbf{Initialize}: initialize $(x_0, y_0)$ arbitrarily.}
	\FOR{$t = 0 \ldots T-1$}
	\STATE{Sample $\xi_{1,t}  \sim N(0, \sigma_1^2 I_{d_1})$ and $\xi_{2,t} \sim N(0, \sigma_2^2 I_{d_2})$.}
	\STATE{$\tilde{x}_t = x_t + \xi_{1,t}, ~ \tilde{y}_t = y_t + \xi_{2,t}.$}
	\STATE{$x_{t+1} = x_t - \eta (\tilde{x}_t \tilde{y}_t^\top - M ) \tilde{y}_t.$}
	\STATE{$y_{t+1} = y_t - \eta (\tilde{y}_t \tilde{x}_t^\top - M^\top ) \tilde{x}_t.$}
	\ENDFOR
    \end{algorithmic}
\end{algorithm}

\end{document}